\documentclass[10pt]{article}

\usepackage{times}
\usepackage[left=1in, right=1in, top=1in, bottom=1in]{geometry}
\usepackage{lineno,hyperref}
\usepackage{xcolor}
\definecolor{mydarkblue}{rgb}{0,0.08,0.45}
\hypersetup{
    colorlinks=true,
    linkcolor=mydarkblue,
    filecolor=mydarkblue,
    urlcolor=mydarkblue,
}

\usepackage{amsmath,amsfonts,bm, amsthm, amssymb}

\def\eqref#1{equation~\ref{#1}}

\def\1{\bm{1}}

\DeclareMathAlphabet{\mathsfit}{\encodingdefault}{\sfdefault}{m}{sl}
\SetMathAlphabet{\mathsfit}{bold}{\encodingdefault}{\sfdefault}{bx}{n}

\newtheorem{proposition}{Proposition}

\DeclareMathOperator{\tr}{Tr}
\DeclareMathOperator{\diag}{diag}

\usepackage[export]{adjustbox}
\usepackage{url}
\usepackage[mode=buildnew]{standalone}%
\usepackage{algorithm}
\usepackage{algpseudocode}
\usepackage{subcaption}
\usepackage{grffile}
\usepackage{graphicx}
\usepackage{array}
\usepackage{tabu}
\usepackage{multirow}
\usepackage{pbox}
\usepackage{ctable} %
\usepackage{mathtools}
\usepackage{multicol}
\usepackage{cleveref}
\usepackage{bm}
\usepackage{makecell}
\usepackage{placeins}
\usepackage{breqn}
\usepackage{balance}

\usepackage{setspace}

\title{Generative Restricted Kernel Machines: A Framework for Multi-view Generation and Disentangled Feature Learning}
\author{Arun Pandey, Joachim Schreurs, Johan A. K. Suykens \\
	Department of Electrical Engineering, ESAT-STADIUS,\\
	KU Leuven. Kasteelpark Arenberg 10, B-3001 Leuven, Belgium\\
	\texttt{\{arun.pandey,joachim.schreurs,johan.suykens\}@esat.kuleuven.be}
}

\begin{document}
\maketitle
\begin{abstract}
This paper introduces a novel framework for generative models based on Restricted Kernel Machines (RKMs) with joint multi-view generation and uncorrelated feature learning, called Gen-RKM. To enable joint multi-view generation, this mechanism uses a shared representation of data from various views. Furthermore, the model has a primal and dual formulation to incorporate both kernel-based and (deep convolutional) neural network based models within the same setting. When using neural networks as explicit feature-maps, a novel training procedure is proposed, which jointly learns the features and shared subspace representation. The latent variables are given by the eigen-decomposition of the kernel matrix, where the mutual orthogonality of eigenvectors represents the learned \emph{uncorrelated} features. Experiments demonstrate the potential of the framework through qualitative and quantitative evaluation of generated samples on various standard datasets.
\end{abstract}

\section{Introduction \label{sec:intro}}

In the past decade, interest in generative models has grown tremendously, finding applications in multiple fields such as, generated art, on-demand video, image denoising~\cite{vincent_stacked_nodate}, exploration in reinforcement learning~ \cite{florensa_automatic}, collaborative filtering~\cite{salakhutdinov_restricted}, in-painting~\cite{Yeh_2017_CVPR} and many more.
Some examples of generative models based on a probabilistic framework with latent variables are Variational Auto-Encoders \cite{kingma_auto-encoding_2013} and Restricted Boltzmann Machines (RBMs) \cite{Smolensky:1986, salakhutdinov_deep}.
More recently proposed models are based on adversarial training such as Generative Adversarial Networks (GANs)~\cite{goodfellow_generative_2014} and its many variants. Furthermore, auto-regressive models such as Pixel Recurrent Neural Networks (PixelRNNs) \cite{oord_pixel_2016} model the conditional distribution of every individual pixel given previous pixels. All these approaches have their own advantages and disadvantages. For example, RBMs perform both learning and Bayesian inference in graphical models with latent variables. However, such probabilistic models must be properly normalized, which requires evaluating intractable integrals over the space of all possible variable configurations~\cite{salakhutdinov_deep}. Currently GANs are considered as the state-of-the-art for generative modeling tasks, producing high-quality images but are more difficult to train due to unstable training dynamics, unless more sophisticated variants are applied.

Many datasets are composed of different representations of the data, also called views. Views can correspond to different modalities such as sounds, images, videos, sequences of previous frames, etc. Although each view could individually be used for learning tasks, exploiting information from all views together could improve the learning quality~\cite{pu_variational_2016,liu_coupled_2016,chen2017multi}. Furthermore, it is among the goals of the latent variable modelling to model the description of data in terms of \emph{uncorrelated} or \emph{independent} components. Some classical examples are Independent Component Analysis; Hidden Markov models~\cite{Rabiner86anintroduction}; Probabilistic Principal Component Analysis (PCA)~\cite{Tipping99probabilisticprincipal}; Gaussian-Process latent variable model~\cite{Lawrence:2005:PNP:1046920.1194904} and factor analysis. Hence, when learning a latent space in generative models, it becomes interesting to find a disentangled representation. Disentangled variables are generally considered to contain interpretable information and reflect distinct factors of variation in the data for e.g. lighting conditions, style, colors, etc. This makes disentangled representations especially interesting for the generation of plausible pseudo-data with certain desirable properties, e.g. generating new chair designs with a certain armrest or new cars with a predefined color.  The definition of disentanglement in the literature is not precise, however many believe that a representation with statistically independent variables is a good starting point~\cite{schmidhuber1992learning,ridgeway2016survey}. Such representations extract information into a compact form which makes it possible to generate samples with specific characteristics~\cite{chen2018isolating,bouchacourt2018multi,tran2017disentangled,chen2016infogan}. Additionally, these representations have been found to generalize better and be more robust against adversarial attacks~\cite{alemi2016deep}.

In this work, we propose a novel generative mechanism based on the framework of Restricted Kernel Machines (RKMs)~\cite{suykens_deep_2017}, called Generative-RKM (Gen-RKM). RKMs yield a representation of kernel methods with visible and hidden units establishing links between Kernel PCA, Least-Squares Support Vector Machines (LS-SVM)~\cite{suykens_least_2002} and RBMs. This framework has a similar energy form as RBMs, though there is a non-probabilistic training procedure where the eigenvalue decomposition plays the role of normalization. Recently, \cite{houthuys_tensor-based_nodate} used this framework to develop tensor-based multi-view classification models and \cite{joachim} showed how kernel PCA fits into this framework. \\

\noindent \textbf{Contributions: 1)} A novel joint multi-view generative model based on the RKM framework where multiple views of the data can be generated simultaneously. \textbf{2)} Two methods are discussed for computing the pre-image of the feature vectors: with the feature map explicitly known or unknown. We show that the mechanism is flexible to incorporate both kernel-based and (deep convolutional) neural network based models within the same setting. \textbf{3)} When using explicit feature maps, we propose a training algorithm that jointly performs the feature-selection and learns the common-subspace representation in the same procedure. \textbf{4)} Qualitative and quantitative experiments demonstrate that the model is capable of generating good quality images of natural objects. Further illustrations on multi-view datasets exhibit the potential of the model. Thanks to the orthogonality of eigenvectors of the kernel matrix, the learned latent variables are uncorrelated. This resembles a disentangled representation, which makes it possible to generate data with specific characteristics.

\section{Related Work}

Latent space models were studied in several other works, where multiple links with disentanglement are made. VAEs~\cite{kingma_auto-encoding_2013} have become a popular framework among different generative models as they provide more theoretically well-founded and stable training than GANs~\cite{goodfellow_generative_2014}. Learning a VAE
amounts to the optimization of an objective balancing the quality of samples that are autoencoded
through a stochastic encoder-decoder pair, measured by the reconstruction error,  while encouraging the latent space to follow a fixed prior distribution, often the Gaussian distribution. In $\beta$-VAEs~\cite{higgins2017beta}, an adjustable hyperparameter $\beta$ is introduced that balances quality of samples and latent space constraints with reconstruction accuracy. The choice of parameter $\beta = 1$ corresponds to the original VAE formulation. Further, they show that with  $\beta > 1$ (more emphasis on the latent variables to be Gaussian distributed) the model is capable of learning a more disentangled latent representation of the data.  In \cite{burgess2018understanding}, the effect of the $\beta$ term is analyzed more in depth. It was suggested that the stronger pressure for the posterior to match the factorised unit Gaussian prior puts extra constraints on the implicit capacity of the latent bottleneck~\cite{higgins2017beta}. Chen et al.~\cite{chen2018isolating} show a decomposition of the variational lower bound that can be used to explain the success of the $\beta$-VAE~\cite{higgins2017beta} in learning disentangled representations. The authors claim that the total correlation, which forces the model to find statistically independent factors in the data distribution, is the most important term in this decomposition. The role of disentanglement was also studied in GANs, where the InfoGAN~\cite{chen2016infogan} is one of the most known works.

The most common approach of joint multimodal/multiview learning with deep neural networks is to share the top of hidden layers in modality specific networks.  Srivastava and Salakhutdinov~\cite{srivastava2012multimodal} proposed a Deep Boltzmann Machine for learning multimodal data. The multimodal DBM learns a joint density model over the space of multimodal inputs by sharing the hidden units of the last layer. Examples of joint multimodal training for VAEs are~\cite{suzuki2016joint,wu2018multimodal}. The work of~\cite{suzuki2016joint} introduced the joint multi-modal VAE, which learns the common distribution using a joint inference network.  The authors use an ELBO objective with two additional divergence terms to minimize the distance between the uni-modal and the multi-modal importance distributions. The MVAE of~\cite{wu2018multimodal} uses a product of experts formulation and sub-sampled training paradigm to solve the multi-modal inference problem.

In contrast to classical VAE architectures, the proposed model introduces an orthogonal interconnection matrix $U$ motivated by the RKM formulation. The model thus finds an `optimal' linear subspace of the latent space given by the eigendecomposition. In this paper, we argue that this orthogonality leads to better disentanglement and generation quality.  \\

The paper is organized as follows. In Section \ref{sec:GRKM}, we discuss the Gen-RKM training and generation mechanism when multiple data sources are available. In Section \ref{sec:featuremaps}, we explain how the model incorporates both kernel methods and neural networks through the use of implicit and explicit feature maps respectively. In Section \ref{sec:exp}, we show experimental results of our model applied on various public datasets. Section \ref{sec:conc} concludes the paper along with directions towards the future work. Further discussions and derivations are given in the Appendix and the Python code is available at \href{https://www.esat.kuleuven.be/stadius/E/software.php}{https://www.esat.kuleuven.be/stadius/E/software.php}.
\section{Generative Restricted Kernel Machines framework \label{sec:GRKM}}
The proposed Gen-RKM framework consists of two phases: a training phase and a generation phase which occurs one after the other.
\subsection{Training phase of the RKM}
Similar to Energy-Based Models (EBMs, see \cite{lecun_learning_2004} for details), the RKM objective function captures  dependencies between variables by associating a scalar energy to each configuration of the variables. Learning consists of finding an energy function in which the observed configurations of the variables are given lower energies than unobserved ones. Note that the schematic representation of Gen-RKM model, as shown in Figure~\ref{fig:schematic_Gen_rkm} is similar to Discriminative RBMs \cite{larochelle_classification} and the objective function $\mathcal{J}_{t}$ (defined below) has an energy form similar to RBMs with additional regularization terms. The latent space dimension in the RKM setting has a similar interpretation as the number of hidden units in a Restricted Boltzmann Machine, where in the specific case of the RKM these hidden units are uncorrelated.

We assume a dataset $ \mathcal{D} = \{ \bm{x}_{i}, \bm{y}_{i} \}_{i=1}^{N} \text{~with~}\bm{x}_{i} \in \mathbb{R}^d $, $\bm{y}_{i} \in \mathbb{R}^p $ consisting of $N$ data points. Here $\bm{y}_{i}$ may represent an additional view of $\bm{x}_{i}$, e.g., an additional image from a different angle, the caption of an image or a class label. We start with the RKM interpretation of Kernel PCA, which gives an upper bound on the equality constrained Least-Squares Kernel PCA objective function~\cite{suykens_deep_2017}. Applying the feature-maps $\bm{\bm{\phi}}_{1}: \Omega_{x}\mapsto \mathcal{H}_{x}$ and $\bm{\bm{\phi}}_{2}:  \Omega_{y}\mapsto \mathcal{H}_{y}$ to the input data points, where $\mathcal{H}_{x},\mathcal{H}_{y}$ are the corresponding Reproducing Kernel Hilbert Spaces (RKHS) of the feature-maps respectively; the training objective function $\mathcal{J}_{t}$ for generative RKM is given by\footnote{For convenience, it is assumed that the feature vectors are centered in the feature space $\Omega_{x},\Omega_{y}$ using $\tilde{\bm{\bm{\phi}}}(\bm{x}) :=\bm{\bm{\phi}}(\bm{x})- \frac{1}{N}\sum_{i=1}^N \bm{\bm{\phi}}(\bm{x}_i)$. Otherwise, a centered kernel matrix could be obtained using \eqref{eq:centering} in \ref{sec:centering}.}:
\begin{equation}
	\label{eq:obj_train}
	\begin{aligned}
	\mathcal{J}_{t} &= \sum_{i=1}^{N} \left\{-\bm{\bm{\phi}}_{1}(\bm{x}_{i})^\top \bm{U}\bm{h}_i - \bm{\phi}_{2}(\bm{y}_{i})^\top \bm{V} \bm{h}_i + \frac{1}{2} \bm{h}_{i}^\top \bm{\Lambda} \bm{h}_i\right\}
	 + \frac{\eta_{1}}{2}\tr(\bm{U}^\top \bm{U}) + \frac{\eta_{2}}{2}\tr(\bm{V}^\top \bm{V}),
	\end{aligned}
\end{equation}
where $ \bm{U} \in \mathbb{R}^{d_{f} \times s} $, $ \bm{V} \in \mathbb{R}^{p_{f} \times s} $ are the unknown interconnection matrices, $\bm{\Lambda} \succ 0 $ the unknown diagonal matrix and $ \bm{h}_{i} \in \mathbb{R}^s $ are the latent variables modeling a common subspace $\mathcal{H} \subseteq \mathcal{H}_{x} \bigoplus \mathcal{H}_{y}$ between the two feature spaces. (see Figure~\ref{fig:schematic_Gen_rkm}). To obtain this objective from LS-SVM formulation see \ref{subsec: Gen-RKM objective function}.
Given $ \eta_{1} > 0 $ and $ \eta_{2} > 0 $ as regularization parameters, the stationary points of $\mathcal{J}_{t}$ are given by:
\begin{equation}\label{eq:2}
\begin{cases}
\frac{\partial \mathcal{J}_{t}}{\partial \bm{h}_{i}}=0  \implies & {\bm{\Lambda} \bm{h}_i =  \bm{U}^\top \bm{\phi}_{1}(\bm{x}_i) +\bm{V}^\top \bm{\phi}_{2}(\bm{y}_i)}, \enskip \forall i %
\\
\frac{\partial \mathcal{J}_{t}}{\partial \bm{U}}=0     \implies  & {\bm{U} = \frac{1}{\eta_{1}} \sum_{i=1}^{N}\bm{\phi}_{1}(\bm{x}_{i})\bm{h}_{i}^\top },                                    \\
\frac{\partial \mathcal{J}_{t}}{\partial \bm{V}} =0     \implies & {\bm{V} = \frac{1}{\eta_{2}} \sum_{i=1}^{N}\bm{\phi}_{2}(\bm{y}_{i})\bm{h}_{i}^\top }.
\end{cases}
\end{equation}

Substituting $\bm{U}$ and $\bm{V}$ in the first equation above, denoting the diagonal matrix $ \bm{\Lambda} =\diag\{\lambda_1,\ldots,\lambda_s\}\in\mathbb{R}^{s\times s} $ with $s \leq N$, yields the following eigenvalue problem:
\begin{equation}\label{eq:sup_KPCA}
{\left[ \frac{1}{\eta_{1}}\bm{K}_{1} + \frac{1}{\eta_{2}}\bm{K}_{2} \right] \bm{H}^\top=\bm{H}^\top \bm{\Lambda}},
\end{equation}

where $ \bm{H} =\big[\bm{h}_1 ,\dots, \bm{h}_N \big]\in \mathbb{R}^{s\times N} $ with $s\leq N$ is the number of selected principal components and $\bm{K}_{1},\bm{K}_{2} \in \mathbb{R}^{N \times N}$ are the kernel matrices corresponding to data sources\footnote{While in the above section we have assumed that only two data sources (namely $\Omega_{x}$ and $\Omega_{y}$) are available for learning, the above procedure could be extended to multiple data-sources. For the $M$ views or data-sources, this yields the training problem: ${\left[\sum_{\ell=1}^{M} \frac{1}{\eta_{\ell}}\bm{K}_{\ell} \right] \bm{H}^\top=\bm{H}^\top \bm{\Lambda}}.$}. Based on Mercer's theorem \cite{mercer_james_functions}, positive-definite kernel functions $k_{1}: \Omega_{x} \times \Omega_{x} \mapsto \mathbb{R}$, $k_{2}: \Omega_{y} \times \Omega_{y} \mapsto \mathbb{R}$  can be defined such that $k_{1}(\bm{x}_{i},\bm{x}_{j})=\langle \bm{\phi}_{1}(\bm{x}_{i}), \bm{\phi}_{1}(\bm{x}_{j})\rangle_{\mathcal{H}_{x}}$, and $k_{2}(\bm{y}_{i},\bm{y}_{j})=\langle \bm{\phi}_{2}(\bm{y}_{i}), \bm{\phi}_{2}(\bm{y}_{j})\rangle_{\mathcal{H}_{y}},~\forall i,j=1,\dots,N$ forms the elements of corresponding kernel matrices. The feature maps $\bm{\phi}_{1}$ and $\bm{\phi}_{2}$, mapping the input data to the high-dimensional feature space (possibly infinite) are implicitly defined by kernel functions. Typical examples of such kernels are given by the Gaussian RBF kernel $k(\bm{x}_i,\bm{x}_j) = e^{-\|\bm{x}_i-\bm{x}_j\|_2^2/(2\sigma^2)}$ or the Laplace kernel $k(\bm{x}_i,\bm{x}_j) =e^{-\|\bm{x}_i-\bm{x}_j\|_2/\sigma}$ just to name a few \cite{Scholkopf2001}. However, one can also define explicit feature maps, still preserving the positive-definiteness of the kernel function by construction \cite{suykens_least_2002}.  Equation \ref{eq:sup_KPCA} corresponds to a kernel PCA operation. In this spirit, we thus find an orthogonal interconnection matrix $U$ that is the optimal linear subspace of the latent space given by the eigendecomposition.
\begin{figure}[ht]
	\centering
	\includegraphics[width=0.7\textwidth]{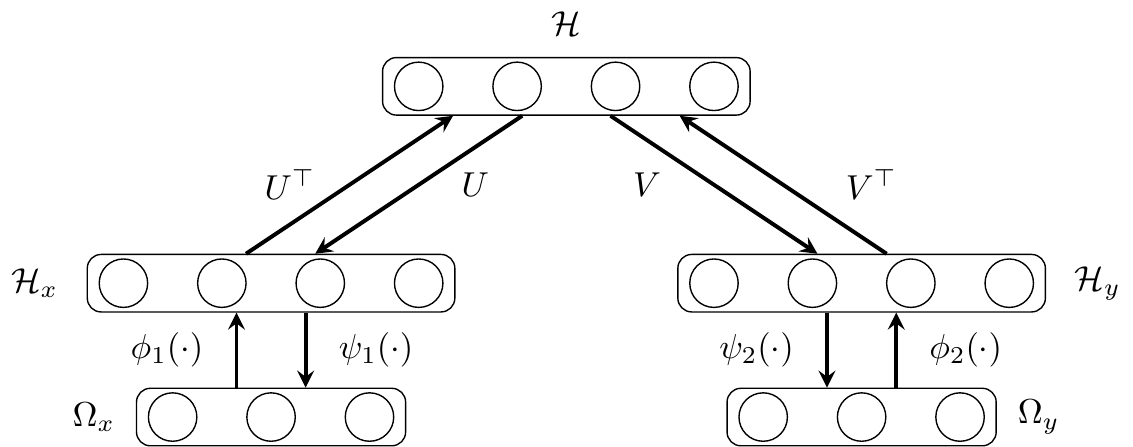}
	\caption{Gen-RKM schematic representation modeling a common subspace $\mathcal{H} \subseteq \mathcal{H}_{x} \bigoplus \mathcal{H}_{y}$ between two data sources $\Omega_{x}$ and $\Omega_{y}$. The $\bm{\phi}_1$, $\bm{\phi}_2$ are the feature maps ($\mathcal{H}_x$ and $\mathcal{H}_y$ represents the RKHS) corresponding to the two data sources. While $\bm{\psi}_1$, $\bm{\psi}_2$ represents the pre-image maps. The interconnection matrices $\bm{U},\bm{V}$ capture the dependencies between latent variables and the mapped data sources.}
	\label{fig:schematic_Gen_rkm}
\end{figure}
\subsection{Generation}
In this section, we derive the equations for the generative mechanism. RKMs resembling energy-based models, the inference consists in clamping the value of observed variables and finding configurations of the remaining variables that minimizes the energy  \cite{lecun_learning_2004}. Given the learned interconnection matrices $\bm U$ and $\bm V$, and a given latent variable $\bm{h}^{\star}$, consider the following generation objective function ${\mathcal{J}_{g}}$:
\begin{equation*}
\begin{aligned}
\mathcal{J}_{g} = - \hat{\bm{\phi}}_{1}(\bm{x}^{\star})^{\top}\bm{U}\bm{h}^{\star} - \hat{\bm{\phi}}_{2}(\bm{y}^{\star})^{\top}\bm{V} \bm{h}^{*} + \dfrac{1}{2}\hat{\bm{\phi}}_{1}(\bm{x}^{\star})^{\top}\hat{\bm{\phi}}_{1}(\bm{x}^{\star})
	 + \dfrac{1}{2}\hat{\bm{\phi}}_{2}(\bm{y}^{\star})^{\top}\hat{\bm{\phi}}_{2}(\bm{y}^{\star}),
\end{aligned}
\end{equation*}
with an additional regularization term on data sources. With slight abuse of notation, we denote the generated feature vectors by $\hat{\bm{\phi}}_{1}(\bm{x}^{\star})$ and $\hat{\bm{\phi}}_{2}(\bm{y}^{\star})$ given the corresponding latent variable $\bm{h}^{\star}$, to distinguish from the feature vectors corresponding to training data points (see \eqref{eq:obj_train}). The given latent variable $\bm{h}^{\star}$ can be the corresponding latent code of a training point, a newly sampled hidden unit or a specifically determined one. Above cases correspond to generating the reconstructed visible unit, generating a random new visible unit or exploring the latent space by carefully selecting hidden units respectively. The stationary points of $\mathcal{J}_{g}$ are characterized by:
\begin{equation}\label{eq:statio_pred}
\begin{cases}
\frac{\partial \mathcal{J}_{g}}{\partial \hat{\bm{\phi}}_{1}(\bm{x}^{\star})}=0     \implies  & \hat{\bm{\phi}}_{1}(\bm{x}^{\star}) = \bm{U}\bm{h}^{\star}, \\
\frac{\partial \mathcal{J}_{g}}{\partial \hat{\bm{\phi}}_{2}(\bm{y}^{\star})} =0     \implies & \hat{\bm{\phi}}_{2}(\bm{y}^{\star}) = \bm{V}\bm{h}^{\star}.
\end{cases}
\end{equation}

Using $\bm{U}$ and $\bm V$ from \eqref{eq:2}, we obtain the generated feature vectors:
\begin{equation}
\label{eq:impl_sol}
\begin{aligned}
\hat{\bm{\phi}}_{1}(\bm{x}^{\star}) = \left(\dfrac{1}{\eta_{1}} \sum_{i=1}^{N}{\bm{\phi}}_{1}(\bm{x}_i)\bm{h}_{i}^\top\right) \bm{h}^{\star},
\quad
\hat{\bm{\phi}}_{2}(\bm{y}^{\star}) = \left(\dfrac{1}{\eta_{2}} \sum_{i=1}^{N}{\bm{\phi}}_{2}(\bm{y}_i)\bm{h}_{i}^\top\right) \bm{h}^{\star}.
\end{aligned}
\end{equation}

To obtain the generated data, now one should compute the inverse images of the feature maps $\hat{\bm{\phi}}_{1}(\cdot)$ and $\hat{\bm{\phi}}_{2}(\cdot)$ in the respective input spaces, i.e., solve the \emph{pre-image problem}. We seek to find the functions $\bm{\psi}_1 \colon \mathcal{H} \mapsto \Omega_{x}$ and $\bm{\psi}_2 \colon \mathcal{H} \mapsto \Omega_{y}$ corresponding to the two data-sources, such that $(\bm{\psi}_1 \circ \hat{\bm{\phi}}_1) (\bm{x}^\star)\approx \bm{x}^\star$ and $(\bm{\psi}_2 \circ \hat{\bm{\phi}}_2) (\bm{y}^\star) \approx \bm{y}^\star$, where $\hat{\bm{\phi}}_1 (\bm{x}^\star)$ and $\hat{\bm{\phi}}_2 (\bm{y}^\star)$ are given using \eqref{eq:impl_sol}.

When using kernel methods, explicit feature maps are not necessarily known. Commonly used kernels such as the radial-basis function and polynomial kernels map the input data to a very high dimensional feature space. Hence finding the pre-image, in general, is known to be an ill-conditioned problem \cite{mika_kernel_nodate}. However, various approximation techniques have been proposed~\cite{bui_projection-free_2019,kwok_pre-image_2004-2,honeine_preimage_2011-1,weston_learning_2004} which could be used to obtain the approximate pre-image $\hat{\bm{x}}$ of $\hat{\bm{\phi}}_{1}(\bm{x}^{\star})$. In section~\ref{sec:implicit}, we employ one such technique to demonstrate the applicability in our model, and consequently generate the multi-view data. One could also define explicit pre-image maps. In section~\ref{sec: explicit}, we define parametric pre-image maps and learn the parameters by minimizing the appropriately defined objective function. The next section describes the above two pre-image methods for both cases, i.e., when the feature map is explicitly known or unknown, in greater detail.

\section{The Proposed Algorithm with Implicit \& Explicit Feature Maps \label{sec:featuremaps}}
\subsection{Implicit feature map}
\label{sec:implicit}
As noted in the previous section, since $ \bm{x}^{\star} $ may not exist, we find an approximation $ \hat{\bm{x}} $. A possible technique is shown by~\cite{joachim}. Left multiplying \eqref{eq:impl_sol} by $\hat{\bm{\phi}}_{1}(\bm{x}_{i})^{\top}$ and $\hat{\bm{\phi}}_{2}(\bm{y}_{i})^{\top}$, $\forall i=1,\dots, N$, we obtain:
\begin{equation}\label{eq:Kxy}
{{\bm k}_{\bm{x}^{\star}}=\frac{1}{\eta_{1}}{\bm{K}}_{1}\bm{H}^\top \bm{h}^{\star}}, \quad    {{\bm{k}}_{\bm{y}^{\star}}=\frac{1}{\eta_{2}}{\bm{K}}_{2}\bm{H}^\top \bm{h}^{\star}},
\end{equation}
where, ${\bm k}_{\bm{x}^{\star}}=\left[ k(\bm{x}_{1}, \bm{x}^{\star}),\dots,k(\bm{x}_{N}, \bm{x}^{\star}) \right]^{\top}$ represents the \emph{similarities} between $\hat{\bm{\phi}}_1 (\bm{x}^{\star})$ and training data points in the feature space, and $\bm{K}_{1}\in \mathbb{R}^{N\times N}$ represents the centered kernel matrix of $\Omega_{x}$. Similar conventions follow for $\Omega_{y}$ respectively. Using the \emph{kernel-smoother} method~\cite{hastie01statisticallearning}, the pre-images are given by:
\begin{equation}
\label{eq: pre-Ix,y}
\begin{aligned}
\hat{\bm{x}} = (\bm{\psi}_1 \circ \hat{\bm{\phi}}_1) (\bm{x}^\star) =\dfrac{\sum_{j=1}^{n_{r}} \tilde{k}_{1}(\bm{x}_{j}, \bm{x}^{\star})\bm{x}_j }{\sum_{j=1}^{n_{r}} \tilde{k}_{1}(\bm{x}_{j}, \bm{x}^{\star})}, \quad  \hat{\bm{y}}= (\bm{\psi}_2 \circ \hat{\bm{\phi}}_2) (\bm{y}^\star) = \dfrac{\sum_{j=1}^{n_{r}} \tilde{k}_{2}(\bm{y}_{j}, \bm{y}^{\star})\bm{y}_j }{\sum_{j=1}^{n_{r}} \tilde{k}_{2}(\bm{y}_{j}, \bm{y}^{\star})},
\end{aligned}
\end{equation}
where $\tilde{k}_1(\bm{x}_i,\bm{x}^{\star})$ and $\tilde{k}_2(\bm{y}_i,\bm{y}^{\star})$ are the scaled similarities (see \eqref{eq: pre-Ix,y}) between $0$ and $1$ and $n_{r}$ the number of closest points based on the similarity defined by kernels $\tilde{k}_1$ and $\tilde{k}_2$.

\subsection{Explicit Feature map}
\label{sec: explicit}

While using an explicit feature map, Mercer's theorem is still applicable due to the positive  semi-definiteness of the kernel function by construction, thereby allowing the derivation of  \eqref{eq:sup_KPCA}. In the experiments, we use a set of (convolutional) neural networks as the parametric feature maps $\bm{\phi}_{\bm{\theta}}(\cdot)$. Another (transposed convolutional) neural network is used for the pre-image map $\bm{\psi}_{\bm{\zeta}}(\cdot)$ \cite{dumoulin2016guide}. The network parameters $\{\bm{\theta},\bm{\zeta}\}$ are learned by minimizing the reconstruction errors $\mathcal{L}_1(\bm{x},\bm{\psi}_{1_{\bm{\zeta}_{1}}} (\hat{\bm{\phi}}_{1_{\bm{\theta}_{1}}} (\bm{x}))) =  \frac{1}{N}\sum_{i=1}^{N} \|\bm{x}_i - \bm{\psi}_{1_{\bm{\zeta}_{1}}} (\hat{\bm{\phi}}_{1_{\bm{\theta}_{1}}} (\bm{x}_i))\|^{2}_2$ for the first view and $\mathcal{L}_2(\bm{y},\bm{\psi}_{2_{\bm{\zeta}_{2}}} (\hat{\bm{\phi}}_{2_{\bm{\theta}_{2}}} (\bm{y}))) =  \frac{1}{N}\sum_{i=1}^{N} \|\bm{y}_i - \bm{\psi}_{2_{\bm{\zeta}_{2}}} (\hat{\bm{\phi}}_{2_{\bm{\theta}_{2}}} (\bm{y}_i))\|^{2}_2$ for the second view, however, in principle, one can use any other loss appropriate to the dataset.
Here $\hat{\bm{\phi}}_{1_{\bm{\theta}_{1}}} (\bm{x}_i)$ and $\hat{\bm{\phi}}_{2_{\bm{\theta}_{2}}} (\bm{y}_i)$ are given by \eqref{eq:impl_sol}, i.e., the generated points in feature space from the subspace $\mathcal{H}$.
Adding the loss function directly into the objective function $\mathcal{J}_t$ is not suitable for minimization. Instead, we use the stabilized objective function defined as $\mathcal{J}_{stab} = \mathcal{J}_{t} + \frac{c_{\mathrm{stab}}}{2} \mathcal{J}_{t}^2 $, where $c_{stab}\in \mathbb{R}^{+}$ is the regularization constant \cite{suykens_deep_2017}. This tends to push the objective function $\mathcal{J}_t$ towards zero, which is also the case when substituting the solutions $\lambda_i , \bm{h}_i$ back into $\mathcal{J}_t$ (see \ref{sec:stabilizationTerm} for details). The combined training objective is given by:
\begin{equation*}
\label{eq:combined_trainigEq}
\begin{aligned}
\min_{\bm{\theta}_1,\bm{\theta}_2,\bm{\zeta}_1,\bm{\zeta}_2} \mathcal{J}_{c} \enskip  = \mathcal{J}_{stab} + \frac{\gamma}{2 N}\bigg(\sum_{i=1}^N \big[ \mathcal{L}_1(\bm{x}_i,\bm{\psi}_{1_{\bm{\zeta}_{1}}} (\hat{\bm{\phi}}_{1_{\bm{\theta}_{1}}} (\bm{x}_i)))  +   \mathcal{L}_2(\bm{y}_i,\bm{\psi}_{2_{\bm{\zeta}_{2}}} (\hat{\bm{\phi}}_{2_{\bm{\theta}_{2}}} (\bm{y}_i))) \big] \bigg),
\end{aligned}
\end{equation*}
where $\gamma\in \mathbb{R}^+$ is a regularization constant to control the stability with reconstruction accuracy. In this way, we combine feature-selection and subspace learning within the same training procedure.

In the objective of the VAE~\cite{kingma_auto-encoding_2013}, an extra term in the form of the Kullback-Leibler divergence between the encoder's distribution and a unit Gaussian is added as a prior on the latent variables. This ensures the latent space is smooth and without discontinuities, which is essential for good generation. By interpreting kernel PCA within the LS-SVM setting~\cite{suykens_least_2002}, the PCA analysis can take the interpretation of a one-class modeling problem with zero target value around which one maximizes the variance~\cite{suykens2003support}. When choosing a good feature map, one expects the latent variables to be normally distributed around zero. As a result, the latent space of the Gen-RKM is continuous, allowing easy random sampling and interpolation (see Figure~\ref{fig:scatter_mnist}).  Kernel PCA gives uncorrelated components in feature space~\cite{bishop_2006}. While the standard PCA does not give a good disentangled representation for images~\cite{eastwood2018a,higgins2017beta}. By designing a good kernel (through appropriate feature-maps) and doing kernel PCA, it is possible to get a disentangled representation for images as we demonstrate in Figure~\ref{fig:uncorre}.

\begin{figure}[h]
	\centering
	\begin{tabular}{cc}
		\setlength\tabcolsep{0pt} %
		\includegraphics[trim={1cm 0.7cm 1.6cm 1cm},clip, height=3.8cm]{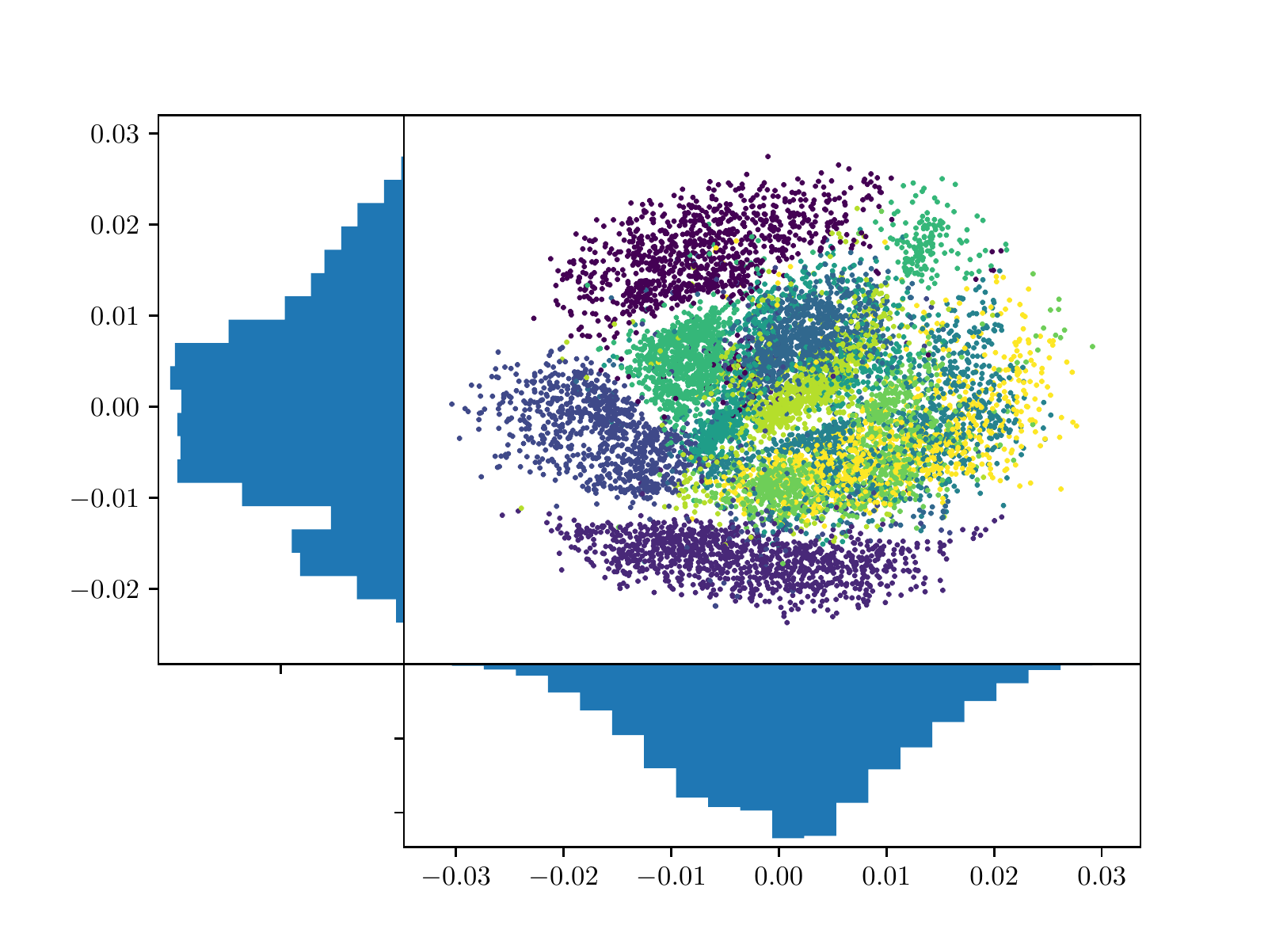} & \includegraphics[trim={0.1cm 0cm 0.1cm 0cm},clip,height=3.7cm]{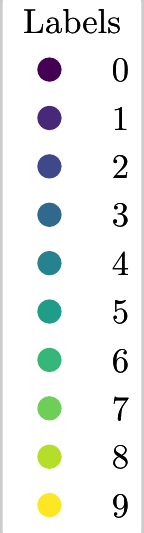}
	\end{tabular}{}
	\caption{MNIST: Scatter plot of latent variable distribution when trained on 10000 images ($s=2$). Training was unsupervised (i.e. one-view) and labels are only used to color the plot. The latent space resembles a Gaussian distribution centered around $0$, where the various digits are clustered together. Generated samples from a uniform grid over this space are shown in Figure~\ref{fig:mnist_unif}.}
	\label{fig:scatter_mnist}
\end{figure}

\subsection{The Gen-RKM Algorithm \label{sec:algo}}

Based on the previous discussion, we propose a novel procedure, called the Gen-RKM algorithm, combining kernel learning and generative models. We show that this procedure is efficient to train and evaluate. The training procedure simultaneously involves feature selection, common-subspace learning and pre-image map learning. This is achieved via an optimization procedure where one iteration involves an eigen-decomposition of the kernel matrix which is composed of the features from various views (see \eqref{eq:sup_KPCA}). The latent variables are given by the eigenvectors, from which a pre-image map reconstructs the generated sample. Figure~\ref{fig:schematic_Gen_rkm} shows a schematic representation of the algorithm when two data sources are available.

Thanks to training in $m$ mini-batches, this procedure is scalable to large datasets (sample size $N$) with training time scaling super-linearly with $ T_m = c \frac{N^{\gamma}}{m^{\gamma -1}}$, instead of $ T_k = c N^{\gamma} $, where $\gamma \approx 3$ for algorithms based on decomposition methods, with some proportionality constant $c$. The training time could be further reduced by computing the covariance matrix (size $(d_{f}+p_{f}) \times (d_{f}+p_{f})$) instead of a kernel matrix (size $\frac{N}{m} \times \frac{N}{m}$), when the sum of the dimensions of the feature-spaces is less than the samples in mini-batch i.e. $d_{f}+p_{f}\leq \frac{N}{m}$. When using neural networks as feature maps, $d_{f}$ and $p_{f}$ correspond to the number of neurons in the output layer, which are chosen as hyperparameters by the practitioner. Eigendecomposition of this smaller covariance matrix would yield $\bm{U}$ and $\bm{V}$ as eigenvectors (see \eqref{eq: cov_mat} and \ref{subsec:cov_mat} for detailed derivation), where computing the $\bm{h}_i$ involves only matrix-multiplication which is readily parallelizable on modern GPUs:
\begin{equation}\label{eq: cov_mat}
	\begin{bmatrix}
		\frac{1}{\eta_1} \bm{\Phi}_{\bm{x}}\bm{\Phi}_{\bm{x}}^{\top} & \frac{1}{\eta_1}\bm{\Phi}_{\bm{x}}\bm{\Phi}_{\bm{y}}^{\top} \\
		\frac{1}{\eta_2}\bm{\Phi}_{\bm{y}}\bm{\Phi}_{\bm{x}}^{\top}  & \frac{1}{\eta_2}\bm{\Phi}_{\bm{y}}\bm{\Phi}_{\bm{y}}^{\top}
	\end{bmatrix}
	\begin{bmatrix}
		\bm{U} \\ \bm{V}
	\end{bmatrix} =
	\begin{bmatrix}
		\bm{U} \\ \bm{V}
	\end{bmatrix}\bm{\Lambda}, \quad
	\begin{aligned}
		\bm{\Phi}_{\bm{x}}  \coloneqq\left[ \bm{\phi}_{1}(\bm{x}_{1}),\dots, \bm{\phi}_{1}(\bm{x}_{N}) \right], \\
		\bm{\Phi}_{\bm{y}}  \coloneqq\left[ \bm{\phi}_{2}(\bm{y}_{1}), \dots, \bm{\phi}_{2}(\bm{y}_{N}) \right].
	\end{aligned}
\end{equation}

\section{Experiments \label{sec:exp}}
To demonstrate the applicability of the proposed framework and algorithm, we trained the Gen-RKM model on a variety of datasets commonly used to evaluate generative models: MNIST \cite{lecun-mnisthandwrittendigit-2010}, Fashion-MNIST \cite{xiao2017/online}, CIFAR-10 \cite{cifar_10}, CelebA \cite{liu2015faceattributes}, Sketchy \cite{sketchy2016}, Dsprites \cite{dsprites17} and Teapot~\cite{eastwood2018a}.
The proposed method adheres to both a primal and dual formulation to incorporate both kernel based methods as well as neural networks based models in the same setting.
The convolutional neural networks are used as explicit feature maps which are known to outperform  kernel based feature maps on the image datasets. Moreover, by using explicit feature maps we demonstrate the capability of the algorithm to jointly learn the feature map and shared subspace representation. For completeness, we also give an illustration when using implicit feature maps through the Gaussian kernel in Figure \ref{fig:multi-view_gen}.
\begin{figure}[]
	\centering
	\begin{subfigure}[b]{0.40\textwidth}
		\centering
		\includegraphics[width=1\textwidth]{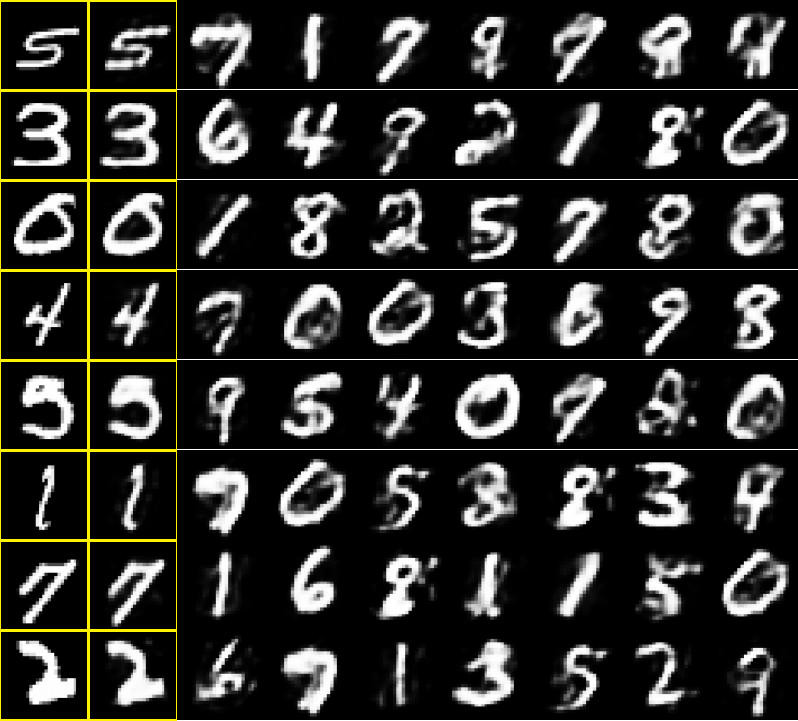}
		\caption{MNIST}
		\label{fig:mnist}
	\end{subfigure}
	\begin{subfigure}[b]{0.40\textwidth}
		\centering
		\includegraphics[width=1\textwidth]{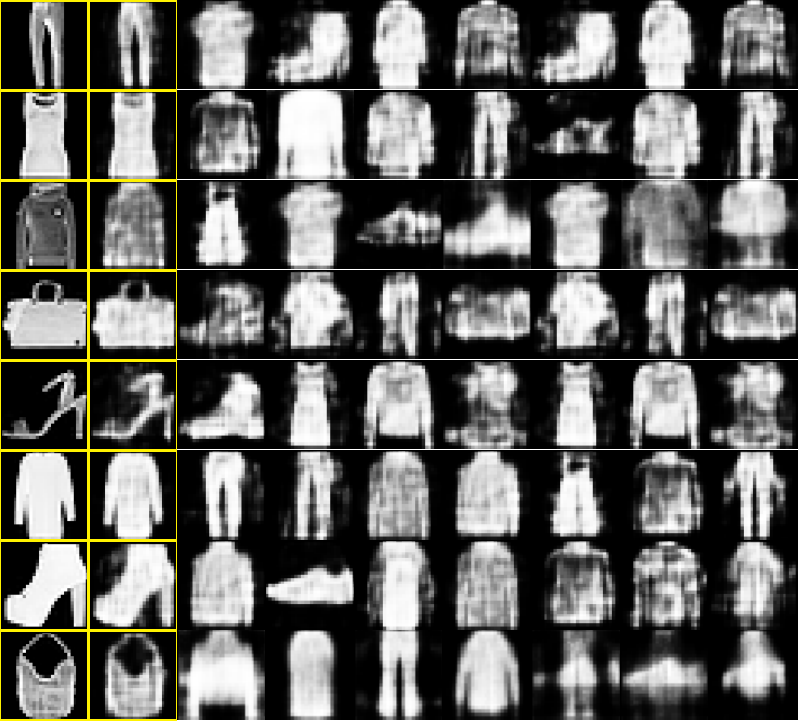}
		\caption{Fashion-MNIST}
		\label{fig:fashion}
	\end{subfigure}
	\\
	\begin{subfigure}[b]{0.40\textwidth}
		\centering
		\includegraphics[width=1\textwidth]{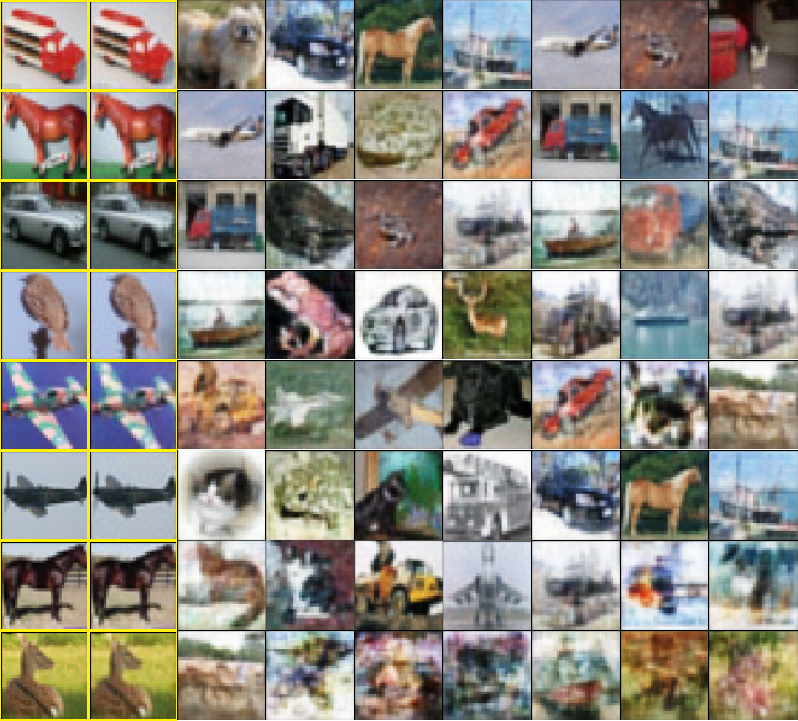}
		\caption{CIFAR-10}
		\label{fig:cifar-10}
	\end{subfigure}
	\begin{subfigure}[b]{0.40\textwidth}
		\centering
		\includegraphics[width=1\textwidth]{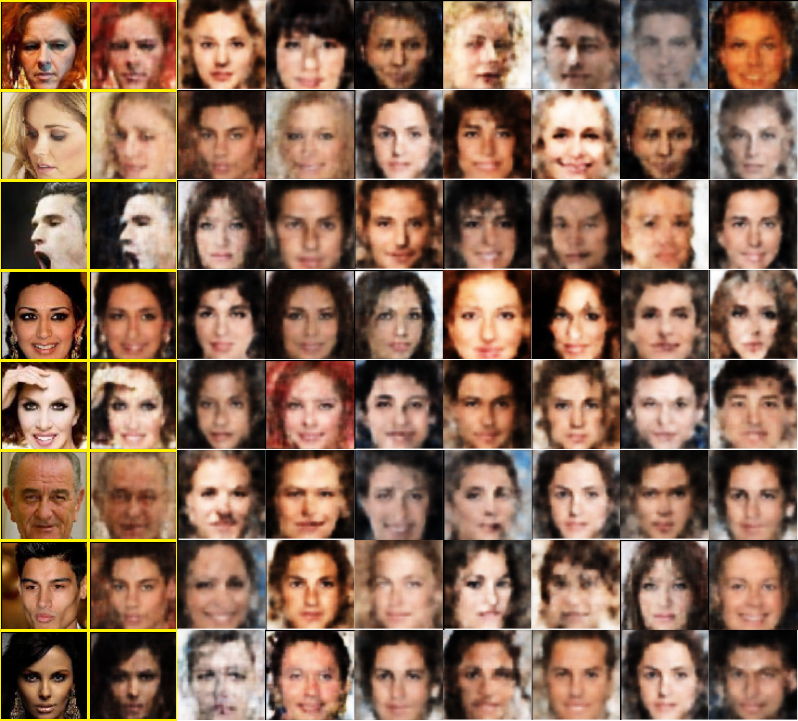}
		\caption{CelebA}
		\label{fig:celeba}
	\end{subfigure}
	\\
	\begin{subfigure}{0.40\textwidth}
		\centering
		\includegraphics[width=1\textwidth]{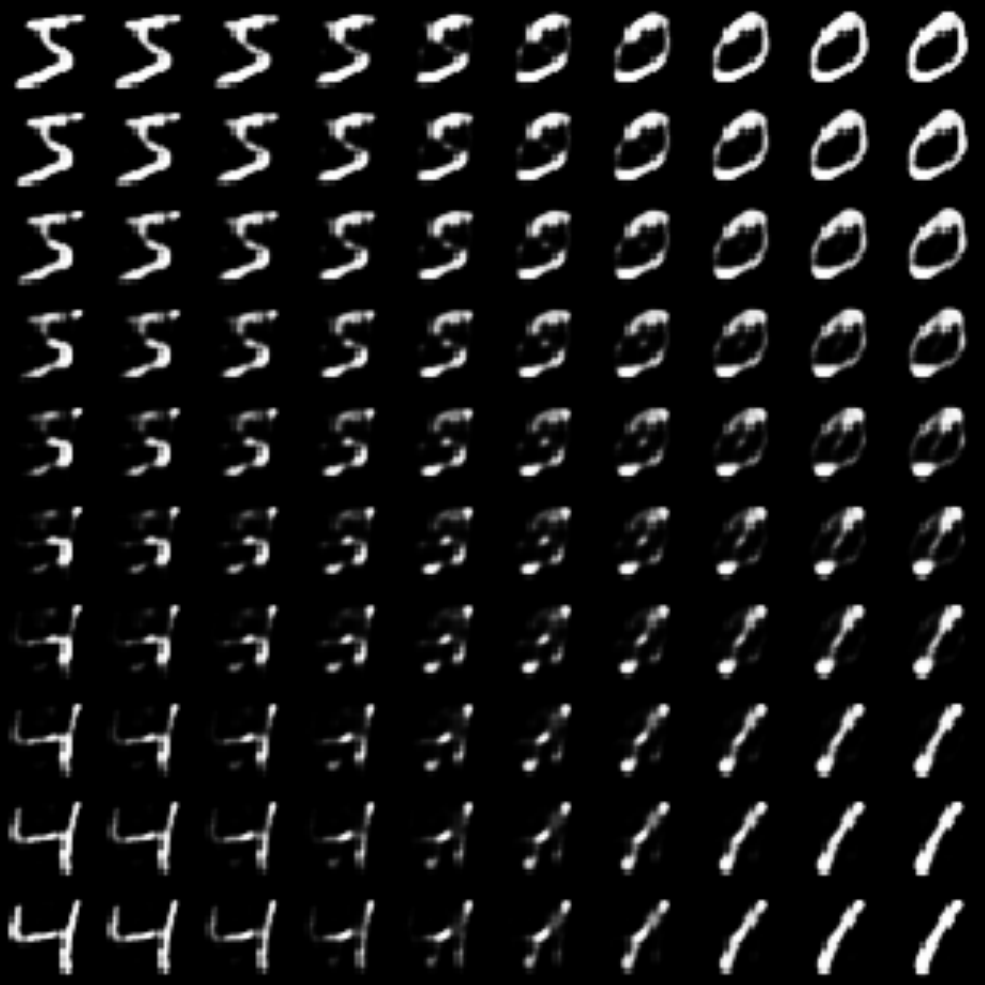}
		\caption{Bilinear interpolation}
		\label{fig:mnist_2d}
	\end{subfigure}
	\begin{subfigure}{0.40\textwidth}
		\centering
		\includegraphics[width=1\textwidth]{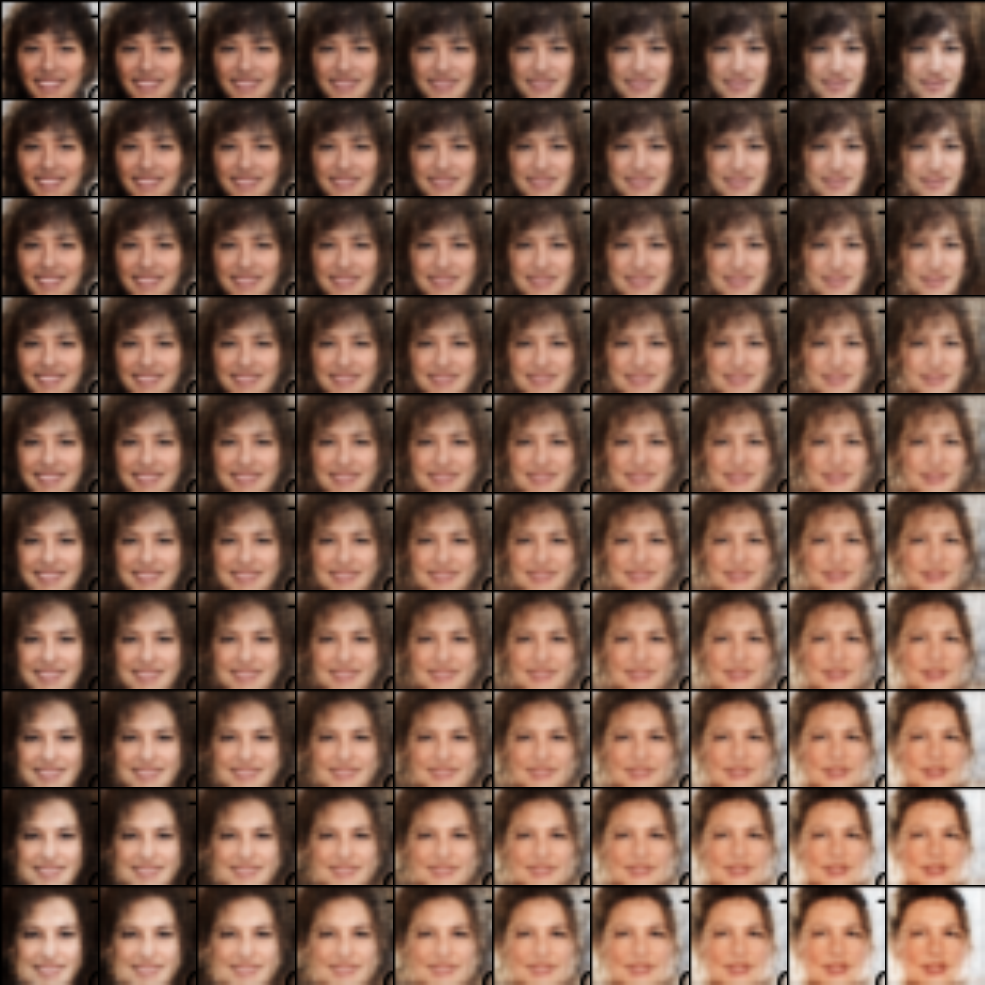}
		\caption{Bilinear interpolation}
		\label{fig:celeb_2d}
	\end{subfigure}
	\caption{Generated samples from the model using CNN as explicit feature map in the kernel function. In (a), (b), (c), (d) the yellow boxes in the first column show training examples and the adjacent boxes show the reconstructed samples. The other images (columns 3-6) are generated by random sampling from the fitted distribution over the learned latent variables. (e) and (f) show the generated images through bilinear interpolations in the latent space.}\label{fig:gen_cov}
\end{figure}

\begin{figure}[]
	\centering
	\resizebox{\textwidth}{!}{
	\begin{tabular}{c c c c}
		 \Huge{Original Images}                                             & \Huge{Images and sketches}                                             & \Huge{Images, sketches and labels} &    \\
		\includegraphics[width=1\textwidth]{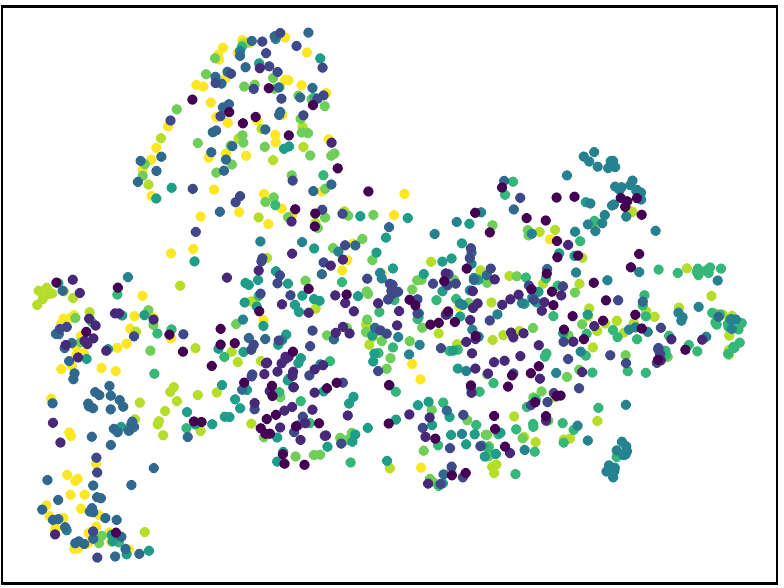}   & \includegraphics[width=1\textwidth]{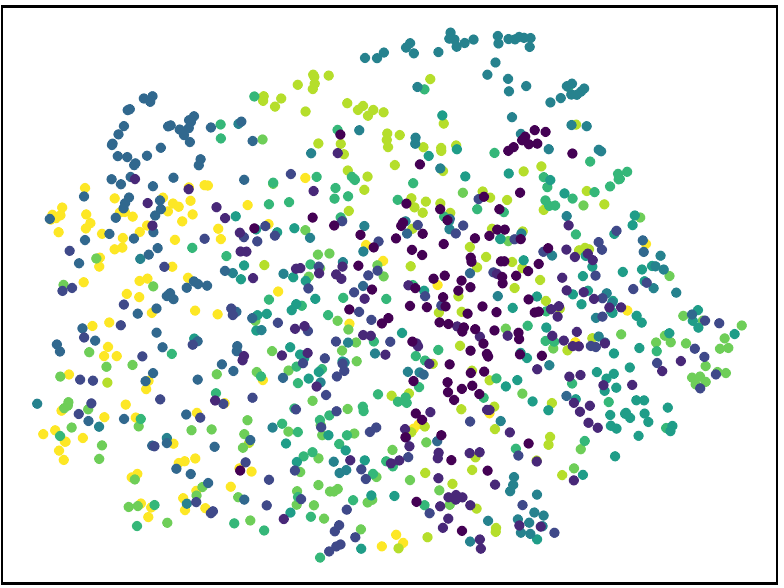} & \includegraphics[width=1\textwidth]{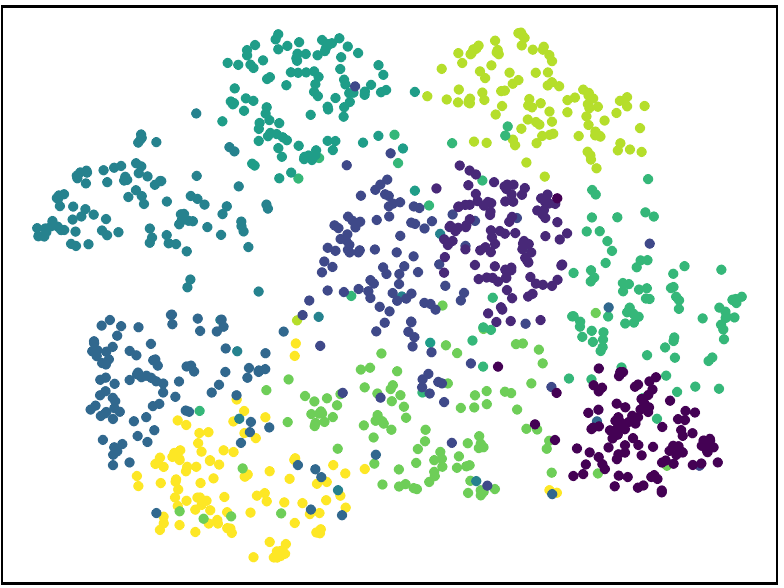} & \includegraphics[width=0.55\textwidth]{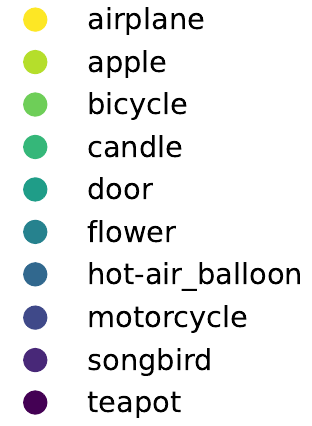}
	\end{tabular}
	}
	\caption{Learned latent space visualization of the Sketchy dataset in 1, 2 and 3-view Gen-RKM setting by using an UMAP embedding.~\cite{McInnes2018}}
	\label{tab:mul_view_scatter}
\end{figure}
\begin{algorithm}[h]
	\caption{Gen-RKM}\label{algo}
	\textbf{Input:} $ \{\bm{x}_{i}, \bm{y}_{i}\}_{i=1}^N,~\eta_{1},~\eta_{2} $, feature map $ \bm{\phi}_j ({\cdot}) $ - explicit \emph{or} implicit via kernels $ k_{j}(\cdot, \cdot), \text{for } j \in \{1, 2\} $ \\
	\textbf{Output:} Generated data $ \bm{x}^{\star},~\bm{y}^{\star}  $
	\begin{multicols}{2}
		\begin{algorithmic}[1]
			\Procedure{Train}{}
			\If{$\bm{\phi}_j ({\cdot})$ = Implicit}
			\State Solve the eigen-decomposition in \eqref{eq:sup_KPCA}
			\State Select the $s$ first principal components
			\ElsIf{$\bm{\phi}_j ({\cdot})$ = Explicit}
			\While{not converged}
			\State $\{\bm{x}, \bm{y}\} \gets \text{\{Get mini-batch\}} $
			\State $ \bm{\phi}_{1}(\bm{x})\gets \bm{x}; ~\bm{\phi}_{2}(\bm{y})\gets \bm{y}  $ 
			\State do steps $ 3\text{-}4 $
			\State $\{\hat{\bm{\phi}}_{1}(\bm{x}), \hat{\bm{\phi}}_{2}(\bm{y})\} \gets \bm{h} $ (\eqref{eq:impl_sol})
			\State $ \{\bm{x}, \bm{y}\} \gets \{\bm{\psi}_1(\hat{\bm{\phi}}_{1}(\bm{x})),  \bm{\psi}_2(\hat{\bm{\phi}}_{2}(\bm{y}))\} $ 
			\State $ \Delta \{\bm{\theta, \zeta}\} \propto -\nabla_{\{\bm{\theta, \zeta}\}} \mathcal{J}_{c} $ 
			\EndWhile
			\EndIf
			\EndProcedure
		\end{algorithmic}
		\columnbreak
		\begin{algorithmic}[1]
			\Procedure{Generation}{}
			\State Select $ \bm{h}^{\star}$
			\If{$\bm{\phi}_j ({\cdot})$ = Implicit}
			\State Set hyperparameter: $ n_{r} $
			\State Compute $ \bm{k}_{\bm{x}^{*}},~\bm{k}_{\bm{y}^{*}} $ (\eqref{eq:Kxy})
			\State $ \text{Get }{\hat{\bm{x}}},~{\hat{\bm{y}}} $ (\eqref{eq: pre-Ix,y})
			\ElsIf{$\bm{\phi}_j ({\cdot})$ = Explicit}
			\State do steps $ 10\text{-}11$
			\EndIf

			\EndProcedure
		\end{algorithmic}
	\end{multicols}
\end{algorithm}
In our experiments, we fit a Gaussian mixture model (GMM) with $l$ components to the latent variables of the training set, and randomly sample a new point $\bm{h}^{\star}$ for generating views. In case of explicit feature maps, we define $\bm{\phi}_{1_{\bm{\theta}_1}}$ and $\bm{\psi}_{1_{\bm{\zeta}_1}}$ as convolution and transposed-convolution neural networks, respectively \cite{dumoulin2016guide}; and $\bm{\phi}_{2_{\bm{\theta}_2}}$ and $\bm{\psi}_{1_{\bm{\zeta}_2}}$ as fully-connected networks. The particular architecture details are outlined in Table~\ref{tab:arch} in the Appendix. The training procedure in case of explicitly defined maps consists of minimizing $\mathcal{J}_{c}$ using the Adam optimizer \cite{Adam} to update the weights and biases. To speed-up learning, we subdivided the datasets into $m$ mini-batches, and within each iteration of the optimizer, \eqref{eq:sup_KPCA} is solved to model the subspace $\mathcal{H}$. Information on the datasets and hyperparameters used for the experiments is given in Table \ref{Table:dataset} in the Appendix. A comparison of the average training time is given in Table \ref{tab:train_times} in the Appendix.

\paragraph{Random Generation}
1) Qualitative examples:
Figure \ref{fig:gen_cov} shows the generated images using a convolutional neural network  and transposed-convolutional neural network as the feature map and pre-image map respectively. The first column in yellow-boxes shows the training samples and the second column on the right shows the reconstructed samples. The other images shown are generated by random sampling from a GMM over the learned latent variables. Notice that the reconstructed samples are of better quality visually than the other images generated by random sampling.  To demonstrate that the model has not merely memorized the training examples, we show the generated images via bilinear-interpolations in the latent space in Figure~\ref{fig:mnist_2d} and  Figure~\ref{fig:celeb_2d}.

\begin{table}[ht]
	\caption{FID Scores~\cite{Heusel2017} for randomly generated samples (smaller is better).}
	\label{Table:fid}
	\centering
	\footnotesize{
	\begin{tabular}{clccc}
		\textbf{Dataset}        & \textbf{Algorithm} & \multicolumn{3}{c}{\textbf{FID score ($\downarrow$)}}                                       \\ \cmidrule{3-5}
		&                    & $s = 10$                          & $s = 30$   & $s = 50$   \\ \midrule
		\multirow{4}{*}{MNIST}  & Gen-RKM            & \textbf{89.825}                         & \textbf{130.497} & \textbf{131.696} \\
		& VAE                & 250                                     & 234.749          & 205.282          \\
		& $\beta$-TCVAE  &    221.45	& 182.93 &	158.31\\
		& InfoGAN     &    238.75 &	204.63 &	179.43 \\ \midrule
		\multirow{4}{*}{CelebA} & Gen-RKM            & \textbf{103.299}                        & \textbf{84.403}  & \textbf{85.121}  \\
		& VAE                & 286.039                                 & 245.738          & 225.783          \\
		& $\beta$-TCVAE   &     248.47 &	226.75 &	173.21\\
		& InfoGAN     &    264.79 &	228.31 &	185.93\\ \midrule
		\multirow{4}{*}{fMNIST} & Gen-RKM  &     \textbf{93.437} &	\textbf{127.893}	& 146.643 \\
		& VAE      &  239.492 &	211.482	& 196.794\\
		& $\beta$-TCVAE       &   206.784 &	187.221 &	\textbf{136.466}\\
		& InfoGAN         &   247.853 &	215.683 &	199.337\\ \midrule
		\multirow{4}{*}{CIFAR10} & Gen-RKM     & \textbf{122.475}	& \textbf{138.467}	& \textbf{158.871} \\
		& VAE         &   295.382 &	259.557 &	231.475\\
		& $\beta$-TCVAE       &  283.58	& 214.681	& 168.483\\
		& InfoGAN         &  295.321	& 258.471	& 220.482\\\bottomrule
	\end{tabular}}
\end{table}

2) Quantitative comparison:
We compare the proposed model with the standard VAE \cite{kingma_auto-encoding_2013}, $\beta$-VAE~\cite{kingma_auto-encoding_2013}, $\beta$-TCVAE~\cite{chen2018isolating} and Info-GAN~\cite{chen2016infogan}. For the Info-GAN, batch normalization is added for training stability. As suggested by the authors, we keep $\alpha=\gamma=1$ and only modify the
hyperparameter $\beta$ for the  $\beta$-TCVAE model.  Determination of the $\beta$ hyperparameter is done by starting from values in the range of the parameters suggested in the authors' reference implementation. After trying various values we noticed that $\beta = 3$  seemed to work good across all datasets that we considered. For a fair comparison, the models have the same encoder/decoder architecture, optimization parameters and are trained until convergence, where the details are given in Table \ref{tab:arch}. We evaluate the performance qualitatively by comparing reconstruction and random sampling, the results are shown in Figure \ref{tab:ReconstructionComparison} in the Appendix. In order to quantitatively assess the quality of the randomly generated samples,  we use the Fr\'echet Inception Distance (FID) introduced by \cite{Heusel2017}. The results are reported in Table \ref{Table:fid}. Experiments were repeated for different latent-space dimensions ($s$), and we observe empirically that FID scores are better for the Gen-RKM. This is confirmed by the qualitative evaluation in Figure~\ref{tab:ReconstructionComparison}. An interesting trend could be noted that as the dimension of latent-space is increased, VAE gets better at generating images whereas the performance of Gen-RKM decreases slightly. This is attributed to the eigendecomposition of the kernel matrix whose eigenvalue spectrum decreases rapidly depicting that most information is captured in few principal components, while the rest is noise. The presence of noise hinders the convergence of the model. It is therefore important to select the number of latent variables proportionally to the size of the mini-batch and the corresponding spectrum of the kernel matrix (the diversity within a mini-batch affects the eigenvalue spectrum of the kernel matrix).

\paragraph{Multi-view Generation}
Figures \ref{fig:sketchy_mul_view} \& \ref{fig:multi-view_gen} demonstrate the multi-view generative capabilities of the model. In these datasets, labels or attributes are seen as another view of the image that provides extra information. One-hot encoding of the labels was used to train the model. Figure \hyperlink{fig:kernel_smooth_mnist}{5a} shows the generated images and labels when feature maps are only implicitly known i.e. through a Gaussian kernel. Figures \hyperlink{fig:mul_view_b}{5b}, \hyperlink{fig:mul_view_c}{5c} shows the same when using fully-connected networks as parametric functions to encode and decode labels. Next we show an illustration of multi-view generation on the Sketchy database~\cite{sangkloy2016sketchy}. The dataset is a collection of sketch-photo pairs resulting in 3 views: images, sketches and labels. The dataset includes 125 object categories with 12,500 natural object
images and 75,471 hand-drawn sketches for each class. The following pre-processing is done before training the GEN-RKM model: for the sketchy dataset, we selected 10 classes for training: airplane, apple, bicycle, candle, door, flower, hot-air-balloon, motorcycle, songbird and teapot. After that, the images are resized to $64 \times 64 \times 3$. Further details on the used model architectures and hyperparameters are given in the Appendix. The learned latent space is visualized on Figure \ref{tab:mul_view_scatter}. One can clearly observe that the joint learning of the different views results in a better separation of the classes. Joint random generations are given in Figure \ref{fig:sketchy_mul_view}.
\begin{figure}[H]
	\centering
	\includegraphics[width=0.70\textwidth]{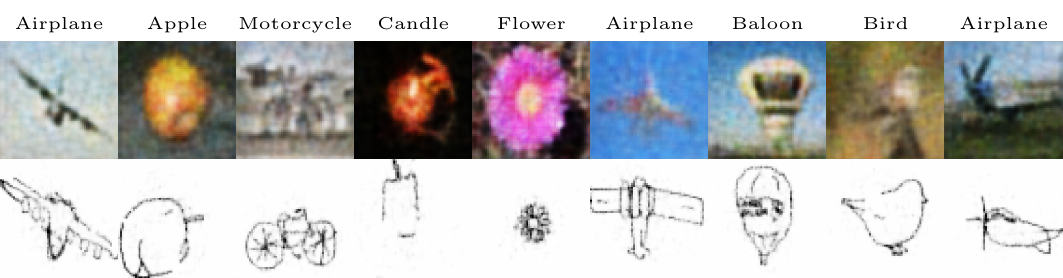}
	\caption{Multi-view generation on Sketchy dataset showing labels, images and sketches generated together from the common subspace.\label{fig:sketchy_mul_view}}
\end{figure}
\begin{figure}[H]
	\centering
	\begin{subfigure}[b]{\textwidth}
		\centering
		\includegraphics[width=0.70\textwidth]{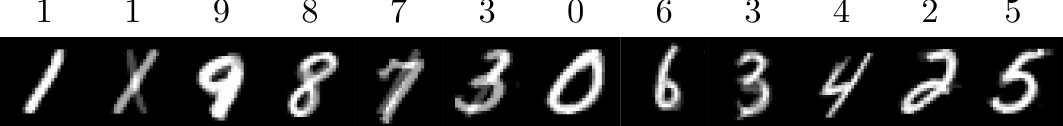}
		\hypertarget{fig:kernel_smooth_mnist}{}
	\end{subfigure}
	\\ \vspace{2pt}
	\begin{subfigure}[b]{\textwidth}
		\centering
		\includegraphics[width=0.70\textwidth]{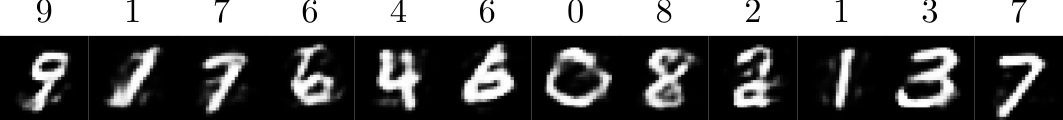}
		\hypertarget{fig:mul_view_b}{}
	\end{subfigure}
	\\ \vspace{2pt}
	\begin{subfigure}[b]{\textwidth}
		\centering
		\includegraphics[width=0.70\textwidth]{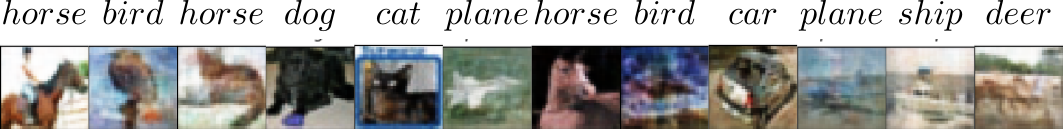}
		\hypertarget{fig:mul_view_c}{}
	\end{subfigure}
	\caption{Multi-view Generation (images and labels) on various datasets using implicit and explicit feature maps. a) MNIST: Implicit feature maps with Gaussian kernel are used during training. For generation, the pre-images are computed using the kernel-smoother method. b,~c) MNIST and CIFAR-10: Explicit feature maps and the corresponding pre-image maps are defined by the Convolutional Neural Networks and Transposed CNNs respectively.}
	\label{fig:multi-view_gen}
\end{figure}
\begin{figure}[H]
	\centering
	\includegraphics[width=0.70\textwidth]{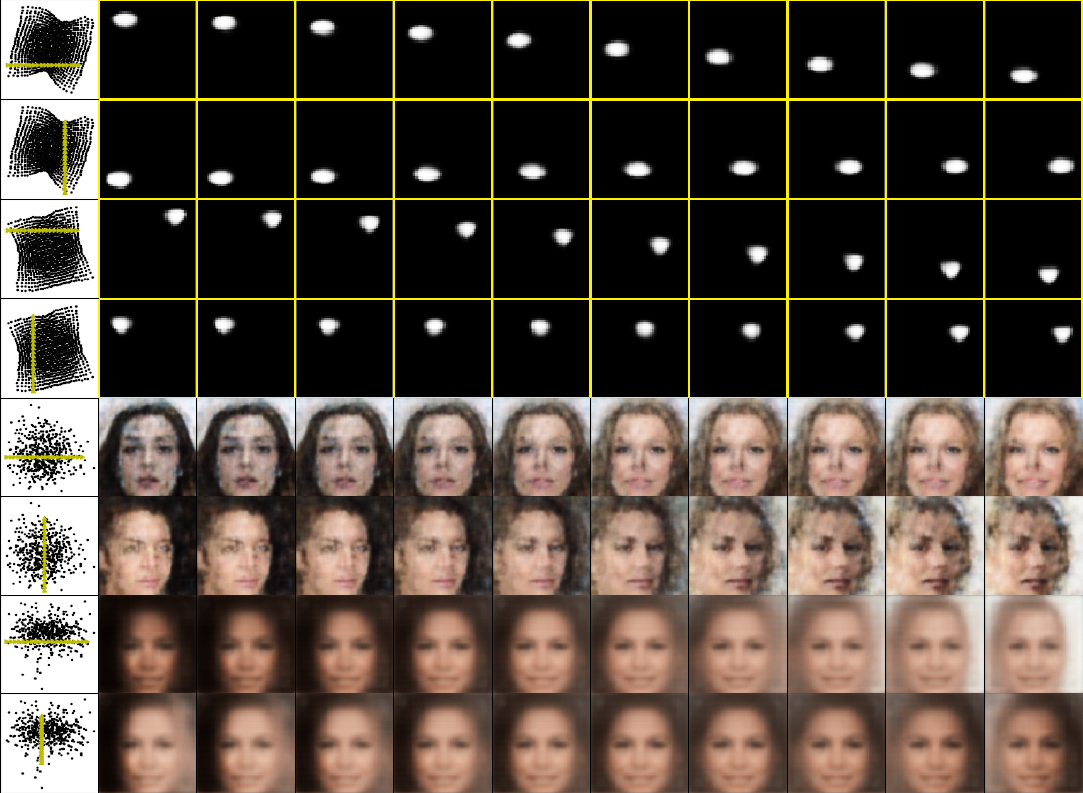}
	\caption{Exploring the learned uncorrelated-features by traversing along the eigenvectors. The first column shows the scatter plot of latent variables using the top two principal components. The green lines within, show the traversal in the latent space and the related rows show the corresponding reconstructed images.}
	\label{fig:uncorre}
\end{figure}
\begin{table}[H]
	\caption{Disentanglement Metric on DSprites and Teapot dataset with Lasso and Random Forest regressor \cite{eastwood2018a}. For disentanglement and completeness higher score is better, for informativeness, lower is better.}
	\label{Table:disang}
	\centering
	\begin{tabular}{cclcccc ccc}
		&  \multirow{2}{*}{s}               &           \multirow{2}{*}{\textbf{Algorithm}}                                & \multicolumn{3}{c}{\textbf{Lasso}} &  & \multicolumn{3}{c}{\textbf{Random Forest}}                                                                             \\
		\cmidrule{4-6} \cmidrule{8-10}
		\multirow{10}{*}{\rotatebox{90}{DSprites\qquad}} &            &                         & \textbf{Dis. ($\uparrow$)}                    & \textbf{Com. ($\uparrow$)}                            & \textbf{Inf. ($\downarrow$)} & & \textbf{Dis. ($\uparrow$)} & \textbf{Com. ($\uparrow$)} & \textbf{Inf. ($\downarrow$)} \\ \midrule
		& \multirow{5}{*}{$10$} & Gen-RKM                                   & 0.30                                & 0.10                                        & 0.87        &      & 0.12             & 0.10             & 0.28             \\
		&                       & VAE                                       & 0.11                                & 0.09                                        & \textbf{0.17} &    & \textbf{0.73}    & \textbf{0.54}    & \textbf{0.06}    \\
		&                       & $\beta \text{-VAE} \left(\beta =3\right)$ & {0.53}                       & \textbf{0.18}                               & 0.18          &    & 0.58             & 0.36             & \textbf{0.06}    \\
		&                       & $\beta\text{-TCVAE}\left(\beta =3\right)$ &   \textbf{0.55} &	0.17 &	0.18 &&	0.72 &	\textbf{0.54}	& 0.11\\
		&                       & Info-GAN &    0.37 &	0.13 &	0.22 &&	0.61 &	0.35 &	0.15\\\cmidrule{2-10}
		& \multirow{5}{*}{$2$}  & Gen-RKM                                   & \textbf{0.72}                       & \textbf{0.71}                               & \textbf{0.64} &    & \textbf{0.05}    & 0.19             & \textbf{0.03}    \\
		&                       & VAE                & 0.04                                & 0.01                                        & 0.87           &   & 0.01             & 0.13             & 0.11             \\
		&                       & $\beta \text{-VAE} \left(\beta =3\right)$ & 0.13                                & 0.40                                        & 0.71          &    & 0.00             & \textbf{0.26}    & 0.09             \\
		&                       & $\beta\text{-TCVAE}\left(\beta =3\right)$ &   0.51 &	0.15 &	0.67 & &	0.03 &	0.17 &	0.14\\
		&                       & Info-GAN &    0.46 &	0.14 &	0.66 &&	0.04 &	0.17 &	0.21\\ \midrule
		\multirow{10}{*}{\rotatebox{90}{Teapot\quad}}
		& \multirow{5}{*}{$10$} & Gen-RKM                                   & 0.28                                & {0.23}                                      & 0.39       &      & \textbf{0.48}    & \textbf{0.39}    & \textbf{0.19}    \\
		&                       & VAE                                       & {0.28}                              & {0.21}                                      & \textbf{0.36}&    & 0.30             & 0.27             & 0.21             \\
		&                       & $\beta \text{-VAE} \left(\beta =3\right)$ & {0.33}                       & \textbf{0.25}                               & \textbf{0.36}&    & 0.31             & 0.24             & 0.20             \\
		&                       & $\beta\text{-TCVAE}\left(\beta =3\right)$ &  \textbf{0.35}	& 0.24 & 0.39 &&	0.35 &	0.25 &	0.31\\
		&                       & Info-GAN &  0.23	& 0.2	& 0.41 &&	0.32 &	0.21 &	0.22\\\cmidrule{2-10}
		& \multirow{5}{*}{$5$}  & Gen-RKM                                   & 0.22                                & 0.23                                        & 0.74          &   & 0.08             & 0.09             & \textbf{0.27}    \\
		&                       & VAE                                       & 0.16                                & 0.14                                        & \textbf{0.66} &   & 0.11             & 0.14             & 0.28             \\
		&                       & $\beta \text{-VAE} \left(\beta =3\right)$ & {0.31}                       & {0.25}                               & 0.68          &   & \textbf{0.13}    & {0.15}    & 0.29             \\
		&                       & $\beta\text{-TCVAE}\left(\beta =3\right)$ &  \textbf{0.33} &	\textbf{0.26} &	0.69 &&	0.12 &	\textbf{0.16} &	0.29\\
		&                       & Info-GAN &    0.21 &	0.19 &	0.71 &&	0.11	& 0.14 &	0.28\\ \midrule
	\end{tabular}
\end{table}

\paragraph{Disentanglement}
1) Qualitative examples:
The latent variables are uncorrelated, which gives an indication that the model could resemble a disentangled representation. This is confirmed by the empirical evidence in Figure~\ref{fig:uncorre}, where we explore the uncorrelated features learned by the models on the Dsprites and celebA datasets. In our experiments, the Dsprites training dataset comprised of $32\times 32$ positions of oval and heart-shaped objects. The number of principal components chosen were 2 and the goal was to find out whether traversing in the direction of principal components, corresponds to traversing the generated images in one particular direction while preserving the shape of the object. Rows 1 and 2 of Figure \ref{fig:uncorre} show the reconstructed images of an oval while moving along first and second principal component respectively. Notice that the first and second components correspond to the $y$ and $x$ positions respectively. Rows 3 and 4 show the same for hearts. On the celebA dataset, we train the Gen-RKM with $15$ components on a subset.  Rows 5 and 6 shows the reconstructed images while traversing along the principal components. When moving along the first component from left-to-right, the hair-color of the women changes, while preserving the face structure. Whereas traversal along the second component, transforms a man to woman while preserving the orientation. When the number of principal components were 2 while training, the brightness and  background light-source corresponds to the two largest variances in the dataset. Also notice that, the reconstructed images are more blurry due to the selection of less number of components to model $\mathcal{H}$.

2) Quantitative comparisons: To quantitatively assess disentanglement performance, we compare Gen-RKM with VAE~\cite{kingma_auto-encoding_2013} and $\beta$-VAE~\cite{higgins2017beta} on the Dsprites and Teapot datasets \cite{eastwood2018a}. The models have the same encoder/decoder architecture, optimization parameters and are trained until convergence, where the details are given in Table \ref{tab:arch}. The performance is measured using the proposed framework\footnote{Code and dataset available at \url{https://github.com/cianeastwood/qedr}} of \cite{eastwood2018a}, which gives $3$ measures: disentanglement, completeness and informativeness. The results are shown in Table \ref{Table:disang}. Gen-RKM has good performance on the Dsprites dataset when the latent space dimension is equal to $2$. This is expected as the number of disentangled generating factors in the dataset is also equal to $2$, hence there are no noisy components in the kernel PCA hindering the convergence. The opposite happens in the case $h_{dim} = 10$, where noisy components are present. The above is confirmed by the Relative Importance Matrix on Figure \ref{fig:hinton_dsp} in the Appendix, where the 2 generating factors are well separated in the latent space of the Gen-RKM. For the Teapot dataset, Gen-RKM has good performance when $s = 10$. More components are needed to capture all variations in the dataset, where the number of generating factors is now equal to $5$. In the other cases, Gen-RKM has a performance comparable to the others. Note that the model selection was done a-priori, that is, the hyperparameters of classifiers were selected before evaluating the disentanglement metric. This may explain the poor scores for Gen-RKM with Random Forest classifier in Teapot dataset ($s=5$).

\section{Conclusion and future work \label{sec:conc}}

The paper proposes a novel framework, called Gen-RKM, for generative models based on RKMs with extensions to multi-view generation and learning uncorrelated representations. This allows for a  mechanism where the feature map can be implicitly defined using kernel functions or explicitly by (deep) neural network based methods. When using kernel functions, the training consists of only solving an eigenvalue problem. In the case of a (convolutional) neural network based explicit feature map, we used (transposed) networks as the pre-image functions. Consequently, a training procedure was proposed which involves joint feature-selection and subspace learning. Thanks to training in mini-batches and capability of working with covariance matrices, the training is scalable to large datasets. Experiments on benchmark datasets illustrate the merit of the proposed framework for generation quality as well as disentanglement.  Extensions of this work consists of adapting the model to more advanced multi-view datasets involving speech, images and texts; further analysis on other feature maps, pre-image methods, loss-functions and uncorrelated feature learning. Finally, this paper has demonstrated the applicability of the Gen-RKM framework, suggesting new research directions to be worth exploring.

\section*{Acknowledgments}
\footnotesize{EU: The research leading to these results has received funding from
the European Research Council under the European Union's Horizon
2020 research and innovation program / ERC Advanced Grant E-DUALITY
(787960). This paper reflects only the authors' views and the Union
is not liable for any use that may be made of the contained information.
Research Council KUL: Optimization frameworks for deep kernel machines C14/18/068
Flemish Government:
FWO: projects: GOA4917N (Deep Restricted Kernel Machines:
Methods and Foundations), PhD/Postdoc grant
Impulsfonds AI: VR 2019 2203 DOC.0318/1QUATER Kenniscentrum Data
en Maatschappij
Ford KU Leuven Research Alliance Project KUL0076 (Stability analysis
and performance improvement of deep reinforcement learning algorithms), ICT 48 TAILOR, Leuven.AI Institute}

\FloatBarrier

\bibliography{references}

\begin{thebibliography}{10}
\expandafter\ifx\csname url\endcsname\relax
  \def\url#1{\texttt{#1}}\fi
\expandafter\ifx\csname urlprefix\endcsname\relax\def\urlprefix{URL }\fi
\expandafter\ifx\csname href\endcsname\relax
  \def\href#1#2{#2} \def\path#1{#1}\fi

\bibitem{vincent_stacked_nodate}
P.~Vincent, H.~Larochelle, I.~Lajoie, Y.~Bengio, P.-A. Manzagol, {Stacked
  Denoising Autoencoders: Learning Useful Representations in a Deep Network
  with a Local Denoising Criterion}, Journal of Machine Learning Research 11
  (2010) 3371--3408.

\bibitem{florensa_automatic}
C.~Florensa, D.~Held, X.~Geng, P.~Abbeel, {Automatic Goal Generation for
  Reinforcement Learning Agents}, in: Proceedings of the 35th International
  Conference on Machine Learning, Vol.~80 of Proceedings of Machine Learning
  Research, {PMLR}, Stockholmsmassan, Stockholm Sweden, 2018, pp. 1515--1528.

\bibitem{salakhutdinov_restricted}
R.~Salakhutdinov, A.~Mnih, G.~Hinton, Restricted {{Boltzmann}} machines for
  collaborative filtering, in: {{ICML}} '07, {ACM Press}, Corvalis, Oregon,
  2007, pp. 791--798.

\bibitem{Yeh_2017_CVPR}
R.~A. Yeh, C.~Chen, T.~Yian~Lim, A.~G. Schwing, M.~Hasegawa-Johnson, M.~N. Do,
  Semantic image inpainting with deep generative models, in: {The IEEE
  Conference on Computer Vision and Pattern Recognition (CVPR)}, 2017.

\bibitem{kingma_auto-encoding_2013}
D.~P. Kingma, M.~Welling, {Auto-Encoding Variational Bayes}, in: 2nd
  International Conference on Learning Representations, {ICLR} 2014, Banff, AB,
  Canada, April 14-16, 2014, Conference Track Proceedings, 2014.

\bibitem{Smolensky:1986}
P.~Smolensky, Parallel distributed processing: Explorations in the
  microstructure of cognition, vol. 1, MIT Press, Cambridge, MA, USA, 1986, Ch.
  Information Processing in Dynamical Systems: Foundations of Harmony Theory,
  pp. 194--281.

\bibitem{salakhutdinov_deep}
R.~Salakhutdinov, G.~Hinton, {Deep Boltzmann Machines}, Proceedings of the 12th
  International Conference on Artificial Intelligence and Statistics Volume 5
  of JMLR (2009).

\bibitem{goodfellow_generative_2014}
I.~J. Goodfellow, J.~Pouget{-}Abadie, M.~Mirza, B.~Xu, D.~Warde{-}Farley,
  S.~Ozair, A.~C. Courville, Y.~Bengio, {Generative Adversarial Nets}, in:
  Advances in Neural Information Processing Systems 27: Annual Conference on
  Neural Information Processing Systems 2014, December 8-13 2014, Montreal,
  Quebec, Canada, 2014, pp. 2672--2680.

\bibitem{oord_pixel_2016}
A.~Van Den~Oord, N.~Kalchbrenner, K.~Kavukcuoglu, {Pixel Recurrent Neural
  Networks}, in: Proceedings of the 33rd International Conference on
  International Conference on Machine Learning - Volume 48, {ICML'16},
  JMLR.org, 2016, pp. 1747--1756.

\bibitem{pu_variational_2016}
Y.~Pu, Z.~Gan, R.~Henao, X.~Yuan, C.~Li, A.~Stevens, L.~Carin, {Variational
  Autoencoder for Deep Learning of Images, Labels and Captions}, {NIPS'16},
  Curran Associates Inc., USA, 2016, pp. 2360--2368.

\bibitem{liu_coupled_2016}
M.-Y. Liu, O.~Tuzel, Coupled {{Generative Adversarial Networks}}, in: Advances
  in {{Neural Information Processing Systems}} 29, {Curran Associates, Inc.},
  2016, pp. 469--477.

\bibitem{chen2017multi}
M.~Chen, L.~Denoyer, Multi-view generative adversarial networks, in: Joint
  European Conference on Machine Learning and Knowledge Discovery in Databases,
  Springer, 2017, pp. 175--188.

\bibitem{Rabiner86anintroduction}
L.~R. Rabiner, B.-H. Juang, {An introduction to Hidden Markov models}, IEEE
  ASSP magazine 3~(1) (1986) 4--16.

\bibitem{Tipping99probabilisticprincipal}
M.~E. Tipping, C.~M. Bishop, Probabilistic principal component analysis,
  Journal Of The Royal Statistical Society, series B 61~(3) (1999) 611--622.

\bibitem{Lawrence:2005:PNP:1046920.1194904}
N.~Lawrence,
  \href{http://dl.acm.org/citation.cfm?id=1046920.1194904}{Probabilistic
  non-linear principal component analysis with gaussian process latent variable
  models}, JMLR 6 (2005) 1783--1816.
\newline\urlprefix\url{http://dl.acm.org/citation.cfm?id=1046920.1194904}

\bibitem{schmidhuber1992learning}
J.~Schmidhuber, Learning factorial codes by predictability minimization, Neural
  Computation 4~(6) (1992) 863--879.

\bibitem{ridgeway2016survey}
K.~Ridgeway, A survey of inductive biases for factorial
  representation-learning, CoRR abs/1612.05299 (2016).
\newblock \href {http://arxiv.org/abs/1612.05299} {\path{arXiv:1612.05299}}.

\bibitem{chen2018isolating}
T.~Q. Chen, X.~Li, R.~B. Grosse, D.~K. Duvenaud, Isolating sources of
  disentanglement in variational autoencoders, in: Advances in Neural
  Information Processing Systems, 2018, pp. 2610--2620.

\bibitem{bouchacourt2018multi}
D.~Bouchacourt, R.~Tomioka, S.~Nowozin, Multi-level variational autoencoder:
  Learning disentangled representations from grouped observations, in:
  Thirty-Second AAAI Conference on Artificial Intelligence, 2018.

\bibitem{tran2017disentangled}
L.~Tran, X.~Yin, X.~Liu, {Disentangled representation learning GAN for
  pose-invariant face recognition}, in: Proceedings of the IEEE Conference on
  Computer Vision and Pattern Recognition, 2017, pp. 1415--1424.

\bibitem{chen2016infogan}
X.~Chen, Y.~Duan, R.~Houthooft, J.~Schulman, I.~Sutskever, P.~Abbeel, Infogan:
  Interpretable representation learning by information maximizing generative
  adversarial nets, in: Advances in neural information processing systems,
  2016, pp. 2172--2180.

\bibitem{alemi2016deep}
A.~Alemi, I.~Fischer, J.~Dillon, K.~Murphy, Deep variational information
  bottleneck, in: ICLR, 2017.

\bibitem{suykens_deep_2017}
J.~A.~K. Suykens, {Deep Restricted Kernel Machines using Conjugate Feature
  Duality}, Neural Computation 29~(8) (2017) 2123--2163.

\bibitem{suykens_least_2002}
J.~A.~K. Suykens, T.~Van~Gestel, J.~De~Brabanter, B.~De~Moor, J.~Vandewalle,
  {Least Squares Support Vector Machines}, {World Scientific}, River Edge, NJ,
  2002.

\bibitem{houthuys_tensor-based_nodate}
L.~Houthuys, J.~A.~K. Suykens, Tensor learning in multi-view kernel {{ PCA }},
  in: 27th International Conference on Artificial Neural Networks {ICANN},
  Rhodes, Greece, Vol. 11140, 2018, pp. 205--215.

\bibitem{joachim}
J.~Schreurs, J.~A.~K. Suykens, Generative {{Kernel PCA}}, in: {{European
  Symposium on Artificial Neural Networks, Computational Intelligence and
  Machine Learning }}, 2018, pp. 129--134.

\bibitem{higgins2017beta}
I.~Higgins, L.~Matthey, A.~Pal, C.~Burgess, X.~Glorot, M.~Botvinick,
  S.~Mohamed, A.~Lerchner, {Beta-VAE}: Learning basic visual concepts with a
  constrained variational framework., ICLR 2~(5) (2017) 6.

\bibitem{burgess2018understanding}
C.~P. Burgess, I.~Higgins, A.~Pal, L.~Matthey, N.~Watters, G.~Desjardins,
  A.~Lerchner, Understanding disentangling in $\backslash$ $beta$-vae, arXiv
  preprint arXiv:1804.03599 (2018).

\bibitem{srivastava2012multimodal}
N.~Srivastava, R.~R. Salakhutdinov, Multimodal learning with deep boltzmann
  machines, in: Advances in neural information processing systems, 2012, pp.
  2222--2230.

\bibitem{suzuki2016joint}
M.~Suzuki, K.~Nakayama, Y.~Matsuo, Joint multimodal learning with deep
  generative models, arXiv preprint arXiv:1611.01891 (2016).

\bibitem{wu2018multimodal}
M.~Wu, N.~Goodman, Multimodal generative models for scalable weakly-supervised
  learning, in: Advances in Neural Information Processing Systems, 2018, pp.
  5575--5585.

\bibitem{lecun_learning_2004}
Y.~LeCun, F.~J. Huang, L.~Bottou, Learning methods for generic object
  recognition with invariance to pose and lighting, in: {Computer Vision and
  Pattern Recognition, 2004. CVPR 2004.}, Vol.~2, 2004, pp. II--97--104 Vol.2.

\bibitem{larochelle_classification}
H.~Larochelle, Y.~Bengio, Classification using discriminative restricted
  {{Boltzmann}} machines, in: Proceedings of the 25th International Conference
  on {{Machine}} Learning - {{ICML}} '08, {ACM Press}, Helsinki, Finland, 2008,
  pp. 536--543.

\bibitem{mercer_james_functions}
J.~Mercer, {Functions of Positive and Negative Type, and Their Connection the
  Theory of Integral Equations}, {Philosophical Transactions of the Royal
  Society of London. Series A, Containing Papers of a Mathematical or Physical
  Character} 209~(441-458) (1909) 415--446.

\bibitem{Scholkopf2001}
B.~Scholkopf, A.~J. Smola, {Learning with Kernels: Support Vector Machines,
  Regularization, Optimization, and Beyond}, MIT Press, Cambridge, MA, USA,
  2001.

\bibitem{mika_kernel_nodate}
S.~Mika, B.~Sch\"{o}lkopf, A.~Smola, K.-R. M\"{u}ller, M.~Scholz,
  G.~R\"{a}tsch, {Kernel PCA and De-noising in Feature Spaces}, in: Proceedings
  of the 1998 Conference on Advances in Neural Information Processing Systems
  II, MIT Press, 1999, pp. 536--542.

\bibitem{bui_projection-free_2019}
A.~T. Bui, J.-K. Im, D.~W. Apley, G.~C. Runger, {Projection-Free Kernel
  Principal Component Analysis for Denoising}, Neurocomputing (2019).

\bibitem{kwok_pre-image_2004-2}
J.~T. Kwok, I.~W.-H. Tsang, The pre-image problem in kernel methods, IEEE
  Transactions on Neural Networks 15 (2003) 1517--1525.

\bibitem{honeine_preimage_2011-1}
P.~Honeine, C.~Richard, Preimage {{Problem}} in {{Kernel}}-{{Based Machine
  Learning}}, {IEEE Signal Processing Magazine} 28~(2) (2011) 77--88.

\bibitem{weston_learning_2004}
J.~Weston, B.~Sch{\"o}lkopf, G.~H. Bakir, Learning to {{Find Pre}}-{{Images}},
  in: NIPS 16, 2004, pp. 449--456.

\bibitem{hastie01statisticallearning}
T.~Hastie, R.~Tibshirani, J.~Friedman, {The Elements of Statistical Learning},
  Springer New York Inc., New York, NY, USA, 2001.

\bibitem{dumoulin2016guide}
V.~Dumoulin, F.~Visin, A guide to convolution arithmetic for deep learning,
  arXiv preprint arXiv:1603.07285 (2016).

\bibitem{suykens2003support}
J.~A.~K. Suykens, T.~Van~Gestel, J.~Vandewalle, B.~De~Moor, A support vector
  machine formulation to {PCA} analysis and its kernel version, IEEE
  Transactions on neural networks 14~(2) (2003) 447--450.

\bibitem{bishop_2006}
C.~M. Bishop, Pattern Recognition and Machine Learning (Information Science and
  Statistics), Springer-Verlag, Berlin, Heidelberg, 2006.

\bibitem{eastwood2018a}
C.~Eastwood, C.~K.~I. Williams,
  \href{https://openreview.net/forum?id=By-7dz-AZ}{A framework for the
  quantitative evaluation of disentangled representations}, in: International
  Conference on Learning Representations, 2018.
\newline\urlprefix\url{https://openreview.net/forum?id=By-7dz-AZ}

\bibitem{lecun-mnisthandwrittendigit-2010}
Y.~LeCun, C.~Cortes, \href{http://yann.lecun.com/exdb/mnist/}{{MNIST}
  handwritten digit database}, http://yann.lecun.com/exdb/mnist/ (2010) [cited
  2016-01-14 14:24:11].
\newline\urlprefix\url{http://yann.lecun.com/exdb/mnist/}

\bibitem{xiao2017/online}
H.~Xiao, K.~Rasul, R.~Vollgraf, {Fashion-MNIST: a Novel Image Dataset for
  Benchmarking Machine Learning Algorithms} (2017).
\newblock \href {http://arxiv.org/abs/cs.LG/1708.07747}
  {\path{arXiv:cs.LG/1708.07747}}.

\bibitem{cifar_10}
A.~Krizhevsky, Learning multiple layers of features from tiny images, Tech.
  rep., {University} of {Toronto} (2009).

\bibitem{liu2015faceattributes}
Z.~Liu, P.~Luo, X.~Wang, X.~Tang, Deep learning face attributes in the wild,
  in: Proceedings of International Conference on Computer Vision (ICCV), 2015.

\bibitem{sketchy2016}
P.~Sangkloy, N.~Burnell, C.~Ham, J.~Hays, The sketchy database: Learning to
  retrieve badly drawn bunnies, ACM Transactions on Graphics (proceedings of
  SIGGRAPH) (2016).

\bibitem{dsprites17}
L.~Matthey, I.~Higgins, D.~Hassabis, A.~Lerchner, dsprites: Disentanglement
  testing sprites dataset, https://github.com/deepmind/dsprites-dataset/
  (2017).

\bibitem{McInnes2018}
L.~McInnes, J.~Healy, N.~Saul, L.~Großberger, Umap: Uniform manifold
  approximation and projection, Journal of Open Source Software 3~(29) (2018)
  861.
\newblock \href {https://doi.org/10.21105/joss.00861}
  {\path{doi:10.21105/joss.00861}}.

\bibitem{Adam}
D.~P. Kingma, J.~Ba, Adam: A method for stochastic optimization, arXiv preprint
  arXiv:1412.6980 (2014).

\bibitem{Heusel2017}
M.~Heusel, H.~Ramsauer, T.~Unterthiner, B.~Nessler, S.~Hochreiter,
  \href{http://dl.acm.org/citation.cfm?id=3295222.3295408}{{GANs} trained by a
  two time-scale update rule converge to a local nash equilibrium}, in:
  Proceedings of the 31st International Conference on Neural Information
  Processing Systems, NIPS'17, Curran Associates Inc., USA, 2017, pp.
  6629--6640.
\newline\urlprefix\url{http://dl.acm.org/citation.cfm?id=3295222.3295408}

\bibitem{sangkloy2016sketchy}
P.~Sangkloy, N.~Burnell, C.~Ham, J.~Hays, The sketchy database: learning to
  retrieve badly drawn bunnies, ACM Transactions on Graphics (TOG) 35~(4)
  (2016) 1--12.

\bibitem{rockafeller1987}
R.~T. Rockafellar, {Conjugate Duality and Optimization}, SIAM, 1974.

\end{thebibliography}

\newpage
\appendix

\normalsize

\section{Derivation of Gen-RKM objective function}
\label{subsec: Gen-RKM objective function}
Given $\mathcal{D}= \{\bm{x}_i, \bm{y}_i\}_{i=1}^{N} $, where $ \bm{x}_i \in \mathbb{R}^d $, $ \bm{y}_i \in \mathbb{R}^p $ and feature-map $\bm{\phi}_{1}: \Omega_{x}\mapsto \mathcal{H}_{x}$ and $\bm{\phi}_{2}: \Omega_{y}\mapsto \mathcal{H}_{y}$, the Least-Squares Support Vector Machine (LS-SVM) formulation of Kernel PCA \cite{suykens_least_2002} for the two data sources can be written as:
\begin{equation}
\begin{aligned}
\min_{\bm{U,V,e}_i} & ~\dfrac{\eta_{1}}{2}\tr(\bm{U}^\top \bm{U} ) + \dfrac{\eta_{2}}{2}\tr(\bm{V}^\top \bm{V}) - \dfrac{1}{2}\sum_{i=1}^{N}\bm{e}_{i}^\top\bm{\Lambda}^{-1} \bm{e}_i \\
\text{s.t.}                   & ~ \bm{e}_i=\bm{U}^{\top}\bm{\phi}_{1}(\bm{x}_i) + \bm{V}^{\top}\bm{\phi}_{2}(\bm{y}_i) \quad \forall i=1,\ldots,N,
\end{aligned}
\label{eq: lssvm_kpca}
\end{equation}
where $ \bm{U} \in \mathbb{R}^{d \times s} $ and $ \bm{V} \in \mathbb{R}^{p \times s} $ are the interconnection matrices.

Using the notion of \emph{conjugate feature duality} introduced in \cite{suykens_deep_2017}, the error variables $\bm{e}_{i}$ are conjugated to latent variables $\bm{h}_{i}$ using:
\begin{equation}
\frac{1}{2}\bm{e}^{\top}\bm{\Lambda}^{-1}\bm{e} + \frac{1}{2}\bm{h}^{\top}\bm{\Lambda}\bm{h} \geq \bm{e}^{\top}\bm{h}, \qquad \forall \bm{e}, \bm{h} \in \mathbb{R}^{s}
\label{eq: conj_feat}
\end{equation}
which is also known as the Fenchel-Young inequality for the case of quadratic functions \cite{rockafeller1987}. By eliminating the variables $\bm{e}_{i}$ from \eqref{eq: lssvm_kpca} and using  \eqref{eq: conj_feat}, we obtain the Gen-RKM training objective function:
\begin{equation*}
\begin{aligned}
\mathcal{J}_{t} = \sum_{i=1}^{N} \left(-\bm{\phi}_{1}(\bm{x}_{i})^\top \bm{U}\bm{h}_i - \bm{\phi}_{2}(\bm{y}_{i})^\top \bm{V} \bm{h}_i + \frac{1}{2} \bm{h}_{i}^\top \bm{\Lambda} \bm{h}_i\right)
 + \frac{\eta_{1}}{2}\tr(\bm{U}^\top \bm{U}) + \frac{\eta_{2}}{2}\tr(\bm{V}^\top \bm{V}) .
\end{aligned}
\end{equation*}

\subsection{Computing latent variables using covariance matrix}
\label{subsec:cov_mat}

From \eqref{eq:2}, eliminating the variables $\bm{h}_{i}$ yields the following:
\begin{equation*}
\begin{aligned}
\frac{1}{\eta_1}\left[ \sum_{i=1}^{N}\bm{\phi}_{1}(\bm{x}_{i})\bm{\phi}_{1}(\bm{x}_{i})^\top \bm{U} + \sum_{i=1}^{N}\bm{\phi}_{1}(\bm{x}_{i})\bm{\phi}_{2}(\bm{y}_{i})^\top \bm{V}  \right] & = \bm{\Lambda}\bm{U}, \\
\frac{1}{\eta_2}\left[  \sum_{i=1}^{N}\bm{\phi}_{2}(\bm{y}_{i})\bm{\phi}_{1}(\bm{x}_{i})^\top \bm{U} + \sum_{i=1}^{N}\bm{\phi}_{2}(\bm{y}_{i})\bm{\phi}_{2}(\bm{y}_{i})^\top \bm{V} \right] & = \bm{\Lambda}\bm{V}.
\end{aligned}
\end{equation*}

Denote $\bm{\Phi}_{\bm{x}}\coloneqq\left[ \bm{\phi}_{1}(\bm{x}_{1}),\dots, \bm{\phi}_{1}(\bm{x}_{N}) \right]$, $ \bm{\Phi}_{\bm{y}}\coloneqq\left[ \bm{\phi}_{2}(\bm{y}_{1}),\dots, \bm{\phi}_{2}(\bm{y}_{N}) \right]$ and the diagonal matrix $ \bm{\Lambda} =\diag\{\lambda_1,\ldots,\lambda_s\}\in\mathbb{R}^{s\times s} $ with $s \leq N$. Now, composing the above equations in matrix form, we get the following eigen-decomposition problem:
\begin{equation*}\label{eq: cov_mat_appen}
\begin{bmatrix}
\frac{1}{\eta_1} \bm{\Phi}_{\bm{x}}\bm{\Phi}_{\bm{x}}^{\top} & \frac{1}{\eta_1}\bm{\Phi}_{\bm{x}}\bm{\Phi}_{\bm{y}}^{\top} \\
\frac{1}{\eta_2}\bm{\Phi}_{\bm{y}}\bm{\Phi}_{\bm{x}}^{\top}  & \frac{1}{\eta_2}\bm{\Phi}_{\bm{y}}\bm{\Phi}_{\bm{y}}^{\top}
\end{bmatrix}
\begin{bmatrix}
\bm{U} \\ \bm{V}
\end{bmatrix} =
\begin{bmatrix}
\bm{U} \\ \bm{V}
\end{bmatrix}\bm{\Lambda}.
\end{equation*}

Here the size of the covariance matrix is $(d_f + p_f)\times (d_f + p_f)$. The latent variables $\bm{h}_{i}$ can be computed using \eqref{eq:2}, which simply involves matrix multiplications.

\section{Stabilizing the objective function \label{sec:stabilizationTerm}}

\begin{proposition}\label{prop2}
	All stationary solutions for $\bm{H}$,$\bm{\Lambda}$ in  \eqref{eq:sup_KPCA} of $ \mathcal{J}_{t} $ lead to $ \mathcal{J}_{t}=0 $.
\end{proposition}
\begin{proof}
	Let $\lambda_i, \bm{h}_i$ are given by \eqref{eq:sup_KPCA}. Using  \eqref{eq:2} to substitute $ \bm{V} $ and $ \bm{U} $  in \eqref{eq:obj_train} yields:
	\[
	\begin{aligned}
		\mathcal{J}_t(\bm{V},\bm{U},\bm{\Lambda},\bm{H})  = & \sum_{i=1}^{N} - \frac{1}{2} \bm{h}_{i}^\top\bm{\Lambda} \bm{h}_i
		+ \frac{\eta_{1}}{2}\tr\left(\dfrac{1}{\eta_{1}^2} \sum_{i=1}^{N} \bm{h}_{i} \bm{\phi}_{1}(\bm{x}_{i})^\top  \sum_{j=1}^{N} \bm{\phi}_{1}(\bm{x}_{j}) \bm{h}_{j}^\top \right) \\
		& + \frac{\eta_{2}}{2}\tr\left(\dfrac{1}{\eta_{2}^2} \sum_{i=1}^{N} \bm{h}_{i} \bm{\phi}_{2}(\bm{y}_{i})^\top  \sum_{j=1}^{N} \bm{\phi}_{2}(\bm{y}_{j}) \bm{h}_{j}^\top \right) \\
		= & \sum_{i=1}^{N} - \frac{1}{2} \bm{h}_{i}^\top \bm{\Lambda}\bm{h}_i + \frac{\eta_{1}}{2}\tr\left(\dfrac{1}{\eta_{1}^2}\bm{H}\bm{K}_{1}\bm{H}^\top \right)
		+ \frac{\eta_{2}}{2}\tr\left(\dfrac{1}{\eta_{2}^2}\bm{H}\bm{K}_{2}\bm{H}^\top\right) \\
		= & \sum_{i=1}^{N} - \frac{1}{2} \bm{h}_{i}^\top\bm{\Lambda} \bm{h}_i + \frac{1}{2}\tr\left(\bm{H} \left[\dfrac{1}{\eta_{1}}\bm{K}_{1} + \dfrac{1}{\eta_{2}}\bm{K}_{2}\right]\bm{H}^\top\right).
	\end{aligned}
	\]
	From  \eqref{eq:sup_KPCA}, we get:
	\[
	\begin{aligned}
		\mathcal{J}_t(\bm{V},\bm{U},\bm{\Lambda},\bm{H}) &= \sum_{i=1}^{N} - \frac{1}{2} \bm{h}_{i}^\top\bm{\Lambda} \bm{h}_i + \frac{1}{2}\tr\left(\bm{H} \bm{H}^\top \bm{\Lambda} \right) \\
		& = \sum_{i=1}^{N} - \frac{1}{2} \bm{h}_{i}^\top\bm{\Lambda} \bm{h}_i + \frac{1}{2}\sum_{i=1}^{N} \bm{h}_{i}^\top\bm{\Lambda} \bm{h}_{i}
		=   0.
	\end{aligned}
	\]
\end{proof}

\begin{proposition}\label{prop1}
	Let $ J(\bm{x}):\mathbb{R}^N\xrightarrow{}\mathbb{R} $ be a smooth function, for all $ \bm{x}\in \mathbb{R}^N$ and for $ c\in \mathbb{R}_{>0} $,   define $ \bar{J}(\bm{x}) := J(\bm{x}) + \dfrac{c}{2}J(\bm{x})^{2} $. Assuming $(1+cJ(\bm{x}))\neq 0$, then $\bm{x}^{\star}$ is the stationary point of $ \bar{J}(\bm{x})$ iff $\bm{x}^{\star}$ is the stationary point for $ {J}(\bm{x})$.
\end{proposition}
\begin{proof}
	Let $ \bm{x}^{\star} $ be a stationary point of $ J(\bm{x}) $, meaning that $ \nabla J(\bm{x}^{\star}) = 0$. The stationary points for $ \bar{J}(\bm{x})$ can be obtained from:
	\begin{equation}\label{eq: stat_pt_jx}
	\dfrac{d \bar{J}}{d \bm{x}} =\left(\nabla J(\bm{x}) + cJ(\bm{x})\nabla J(\bm{x})\right)
	= \left(1+c J(\bm{x})\right)\nabla J(\bm{x}).
	\end{equation}
	It is easy to see from \eqref{prop1} that if $ \bm{x}=\bm{x}^{*} $, $ \nabla J(\bm{x}^*) = 0$, we have that $ \dfrac{d \bar{J}}{d \bm{x}}\Big|_{\bm{x}^*}=0$, meaning that all the stationary points of  $ J(\bm{x}) $ are stationary points of $ \bar{J}(\bm{x}) $.

	To show the other way, let $\bm{x}^{\star}$ be stationary point of $\bar{J}(\bm{x})$ i.e. $\nabla \bar{J}(\bm{x}^{\star})=0.$ Assuming $(1+cJ(\bm{x}^{\star}))\neq 0$, then from \eqref{eq: stat_pt_jx} for all $ c\in \mathbb{R}_{>0}$, we have
	\begin{equation*}
	\left(1+cJ(\bm{x}^{\star})\right)\nabla J(\bm{x}^{\star})=0,
	\end{equation*}
	implying that $ \nabla J(\bm{x}^{\star}) = 0$.
\end{proof}

Based on the above propositions, we stabilize our original objective function \eqref{eq:obj_train} to keep it bounded and hence is suitable for minimization with Gradient-descent methods. Without the reconstruction errors, the stabilized objective function is
\begin{equation*}
\min_{\bm{U},\bm{V},\bm{h}_{i}}\mathcal{J}_{t} + \frac{c}{2}\mathcal{J}_{t}^2.
\end{equation*}
Denoting $ \bar{J} = \mathcal{J}_{t} + \frac{c_{stab}}{2}\mathcal{J}_{t}^2 $. Since the derivatives of $ \mathcal{J}_{t} $ are given by  \eqref{eq:2}, the stationary points of $ \bar{J} $ are:
\begin{equation*}
\begin{cases}
\frac{\partial \bar{J}}{\partial \bm{V}} = \left(1 + c_{stab}\mathcal{J}_{t} \right)\left( - \sum_{i=1}^{N}\bm{\phi}_{1}(\bm{x}_{i})\bm{h}_{i}^\top + \eta_{1}\bm{V}\right) = 0                       , \\
\frac{\partial \bar{J}}{\partial \bm{U}} = \left(1 + c_{stab}\mathcal{J}_{t} \right)\left( - \sum_{i=1}^{N}\bm{\phi}_{2}(\bm{y}_{i})\bm{h}_{i}^\top + \eta_{2}\bm{U}\right) = 0                       , \\
\frac{\partial \bar{J}}{\partial \bm{h}_{i}} = \left(1 + c_{stab}\mathcal{J}_{t} \right)\left( -  \bm{V}^\top \bm{\phi}_{1}(\bm{x}_i) - \bm{U}^\top \bm{\phi}_{2}(\bm{y}_i) + \lambda \bm{h}_i \right) = 0,
\end{cases}
\end{equation*}
which gives the following solution:
\begin{equation*}
\begin{cases}
 {\bm{V} = \frac{1}{\eta_{1}} \sum_{i=1}^{N}\bm{\phi}_{1}(\bm{x}_{i})\bm{h}_{i}^\top }, \\
{\bm{U} = \frac{1}{\eta_{2}} \sum_{i=1}^{N}\bm{\phi}_{2}(\bm{y}_{i})\bm{h}_{i}^\top }, \\
\lambda \bm{h}_i  =  \bm{V}^\top \bm{\phi}_{1}(\bm{x}_i) + \bm{U}^\top \bm{\phi}_{2}(\bm{y}_i),
\end{cases}
\end{equation*}
assuming $ 1 + c_{stab}\mathcal{J}_{t} \neq 0 $. Elimination of $ \bm{V} ~ \text{and} ~ \bm{U} $ yields the following eigenvalue problem ${\left[ \frac{1}{\eta_{1}}\bm{K}_{1} + \frac{1}{\eta_{2}}\bm{K}_{2} \right] \bm{H}^\top=\bm{H}^\top \bm{\Lambda}} $, which is indeed the same solution for $ c_{stab}=0 $ in \eqref{eq:obj_train} and \eqref{eq:sup_KPCA}.

\section{Centering of kernel matrix \label{sec:centering}}

Centering of the kernel matrix is done by the following equation:
\begin{equation}\label{eq:centering}
\bm{K}_{c}= \bm{K} - N^{-1}\bm{11}^{\top} \bm{K} - N^{-1}\bm{K} \bm{11}^{\top} + N^{-2}\bm{11}^{\top}\bm{K}\bm{11}^{\top},
\end{equation}
where $\bm 1$ denotes an $N$-dimensional vector of ones and $ \bm{K} $ is either $ \bm{K}_{1} $ or $ \bm{K}_{2} $.

\section{Architecture details}
\label{subsec: Architecture}

See Table \ref{Table:dataset} and \ref{tab:arch} for details on model architectures, datasets and hyperparameters used in this paper and double precision is used for training the Gen-RKM model. The PyTorch library in Python was used as the programming language with a 8GB NVIDIA QUADRO P4000 GPU.

\begin{table}[H]
	\caption{Datasets and hyperparameters used for the experiments. The bandwidth of the Gaussian kernel for generation corresponds to the bandwidth that gave the best performance determined by cross-validation on the MNIST classification problem.
	}
	\label{Table:dataset}
	\centering
	\begin{tabular}{lccccccccc}
		\toprule
		Dataset       & $N$    & $d$                       & $N_{\mathrm{subset}}$ & $s$  & $m$   & $\sigma$ & $n_r$ & $l$ \\ \midrule
		MNIST         & 60000  & $28 \times 28$            & 10000                  & 500  & 50    & 1.3      & 4     & 10  \\
		Fashion-MNIST & 60000  & $28 \times 28$            & 500                   & 100  & 5     & /        & /     & 10  \\
		CIFAR-10      & 60000  & $32 \times 32 \times 3$   & 500                   & 500  & 5     & /        & /     & 10  \\
		CelebA        & 202599 & $128 \times 128 \times 3$ & 3000                   & 15   & 5     & /        & /     & 20  \\
		Dsprites      & 737280 & $64 \times 64$            & 1024                  & 2/10 & 5     & /        & /     & /   \\
		Teapot        & 200000 & $64 \times 64 \times 3$            & 1000                  & 5/10 & 100 & /        & /     & /   \\
		Sketchy        & 75471 & $64 \times 64 \times 3$            & 1000                  & 30 & 100 & /        & /     & /   \\ \bottomrule
	\end{tabular}
\end{table}

\begin{table}[H]    \caption{Details of model architectures used in the paper. All convolutions and transposed-convolutions are with stride 2 and padding 1. Unless stated otherwise, the layers have Parametric-RELU ($\alpha = 0.2$) activation function, except the output layers of the pre-image maps which has sigmoid activation function. \label{tab:arch}}
	\centering
\begin{tabular}{lclll}
		\textbf{Dataset}          & \multicolumn{1}{c}{\textbf{Optimizer}}         &                    & \multicolumn{2}{c}{\textbf{Architecture}}                                       \\
		& \multicolumn{1}{c}{(Adam)} &                    & \multicolumn{1}{c}{$ \mathcal{X} $}       & \multicolumn{1}{c}{$ \mathcal{Y} $} \\
		\toprule
		\multirow{6}{*}{\makecell{MNIST/\\ fMNIST}}    & \multirow{6}{*}{1e-3}      & Input              & 28x28x1                                   & 10 (One-hot encoding)               \\
		&                            & Feature-map (fm)   & \pbox{20cm}{Conv 32x4x4;                                                        \\  Conv 64x4x4;\\ FC 128 (Linear)} &      FC 15, 20 (Linear)                 \\
		&                            & Pre-image map      & reverse of fm                             & reverse of fm                       \\
		&                            & Latent space dim.  & \multicolumn{2}{c}{$500/100$}                                                       \\
		\midrule
		\multirow{6}{*}{CIFAR-10} & \multirow{6}{*}{1e-3}      & Input              & 32x32x3                                   & 10 (One-hot encoding)               \\
		&                            & Feature-map (fm)   & \pbox{20cm}{Conv 64x4x4;                                                        \\  Conv 128x4x4;\\ FC 128 (Linear)}                                   &             FC 15, 20                \\
		&                            & Pre-image map      & reverse of fm                             & reverse of fm                       \\
		&                            & Latent space dim.  & \multicolumn{2}{c}{$500$}                                                       \\
		\midrule
		\multirow{6}{*}{CelebA}   & \multirow{6}{*}{1e-4}      & Input              & 64x64x3                                   & -                                   \\
		&                            & Feature-map (fm)   & \makecell[l]{Conv 32x4x4;                                                        \\ Conv 64x4x4;\\ Conv 128x4x4;\\Conv 256x4x4 ;\\ FC 128 (Linear)}                                         &            -                \\
		&                            & Pre-image map      & reverse of fm                             & -                                   \\
		&                            & Latent space dim.  & \multicolumn{2}{c}{$15$}                                                        \\
		\midrule
		\multirow{6}{*}{Dsprites} & \multirow{6}{*}{1e-4}      & Input              & 64x64x1                                   & -                                   \\
		&                            & Feature-map (fm)   & \pbox{20cm}{Conv 20x4x4;                                                        \\ Conv 40x4x4;\\ Conv 80x4x4;\\ FC 128 (Linear)}                                        &           -            \\
		&                            & Pre-image map      & reverse of fm                             & -                                   \\
		&                            & Latent space dim.  & \multicolumn{2}{c}{$2/10$}                                                      \\
		\midrule
		\multirow{6}{*}{Teapot}   & \multirow{6}{*}{1e-4}      & Input              & 64x64x3                                   & -                                   \\
		&                            & Feature-map (fm)   & \pbox{20cm}{Conv 30x4x4;                                                        \\ Conv 60x4x4;\\ Conv 90x4x4;\\ FC 128 (Linear)}                                        &           -            \\
		&                            & Pre-image map      & reverse of fm                             & -                                   \\
		&                            & Latent space dim.  & \multicolumn{2}{c}{$5/10$}                                                      \\
		\midrule
		\multirow{6}{*}{Sketchy}   & \multirow{6}{*}{1e-4}      & Input              & 64x64x3                                   & -                                   \\
		&                            & Feature-map (fm)   & \pbox{20cm}{Conv 40x4x4;                                                        \\ Conv 80x4x4;\\ Conv 160x4x4;\\ FC 128 (Linear)}                                        &           -            \\
		&                            & Pre-image map      & reverse of fm                             & -                                   \\
		&                            & Latent space dim.  & \multicolumn{2}{c}{$30$}                                                      \\
		\bottomrule
	\end{tabular}
\end{table}

\section{Bilinear Interpolation}\label{appen:bilin_interpo}
Given four vectors $\bm{h}_1 , \bm{h}_2 , \bm{h}_3 \text{ and } \bm{h}_4 $ (reconstructed images from these vectors are shown at the corners of Figures \ref{fig:mnist_2d}, \ref{fig:celeb_2d}),  the interpolated vector $\bm{h}^{\star}$ is given by:
\[
\bm{h}^{\star} = (1-\alpha)(1-\gamma)\bm{h}_1 + \alpha (1-\gamma)\bm{h}_2 + \gamma (1-\alpha)\bm{h}_3 + \gamma \alpha\bm{h}_4,
\]
with $\quad 0\leq \alpha, \gamma \leq 1.$
This $\bm{h}^{\star}$ is then used in step 8 of the generation procedure of Gen-RKM algorithm (see Algorithm \ref{algo}) to compute $\bm{x}^{\star}$.

\section{Visualizing the disentanglement metric}
In this section we show the Hinton plots to visualize the disentaglement scores as shown in Table \ref{Table:disang}. Following the conventions of \cite{eastwood2018a}, $\bm{z}$ represents the ground-truth data generating factors. Figures \ref{fig:hinton_dsp} and \ref{fig:hinton_tea} shows the Hinton plots on DSprites and Teapot datasets using Lasso and Random Forest regressors for various algorithms. Here the white square size indicates the magnitude of the \emph{relative importance} of the latent code $\bm{h}_{i}$ in predicting $\bm{z}_{i}$.

\begin{figure}[H]
	\begin{subfigure}{1\textwidth}
		\centering
		 \resizebox{\textwidth}{!}{
		\begin{tabular}{c c c c| c c c c}
			\multicolumn{4}{c|}{\textbf{Lasso}}                                       & \multicolumn{4}{c}{\textbf{Random Forest}}                                                                                                                                                                                                                                                                                                                                                                                                \\
			\multicolumn{1}{c}{Gen-RKM}                                               & \multicolumn{1}{c}{VAE}                                                   & \multicolumn{1}{c}{$\beta \text{-VAE}$}                                   & \multicolumn{1}{c}{InfoGAN}                                   &\multicolumn{1}{c}{Gen-RKM}                                                               & \multicolumn{1}{c}{VAE}                                                                   & \multicolumn{1}{c}{$\beta \text{-VAE}$}              & \multicolumn{1}{c}{InfoGAN}                   \\
			\includegraphics[height=5cm]{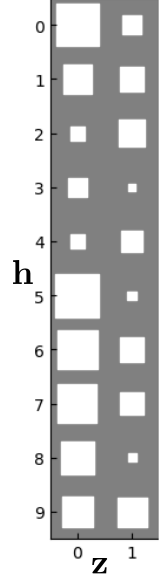} & \includegraphics[height=5cm]{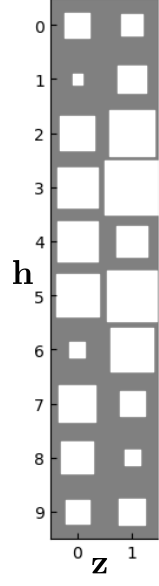} & \includegraphics[height=5cm]{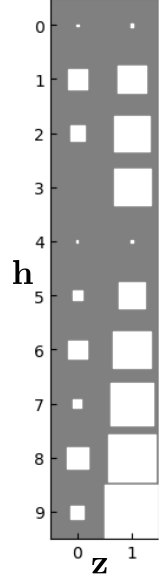} & \includegraphics[height=5cm]{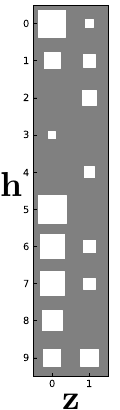} &\includegraphics[height=5cm]{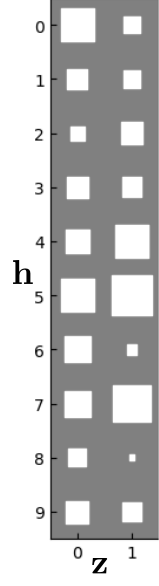} & \includegraphics[height=5cm]{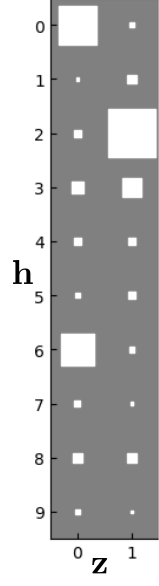} & \includegraphics[height=5cm]{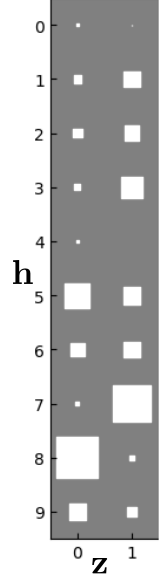} & \includegraphics[height=5cm]{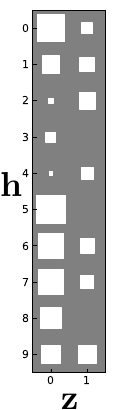}\\

		\end{tabular}
		}
		\caption{${h}_{dim}=10$}
	\end{subfigure}

	\begin{subfigure}{1\textwidth}
		\centering
		 \resizebox{\textwidth}{!}{
		\begin{tabular}{ c c c c| c c c c}
			\multicolumn{4}{c|}{\textbf{Lasso}}                                         & \multicolumn{4}{c}{\textbf{Random Forest}}                                                                                                                                                                                                                                                                                                                                                                                                          \\
			\multicolumn{1}{r}{Gen-RKM}                                                 & \multicolumn{1}{c}{VAE}                                                     & \multicolumn{1}{c}{$\beta \text{-VAE}$}   & \multicolumn{1}{c}{InfoGAN}                                        & \multicolumn{1}{r}{Gen-RKM}                                                                 & \multicolumn{1}{c}{VAE}                                                                     & \multicolumn{1}{c}{$\beta \text{-VAE}$}    & \multicolumn{1}{c}{InfoGAN}                                                      \\
			\includegraphics[height=1.2cm]{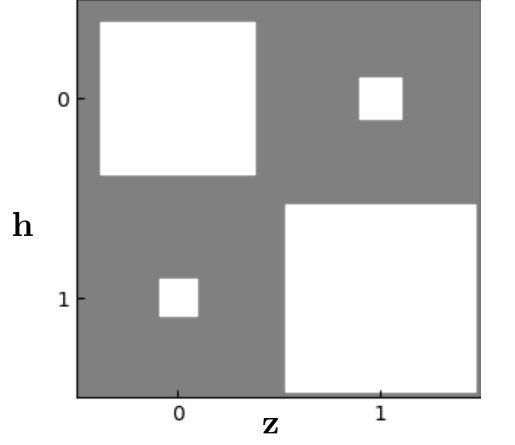} & \includegraphics[height=1.2cm]{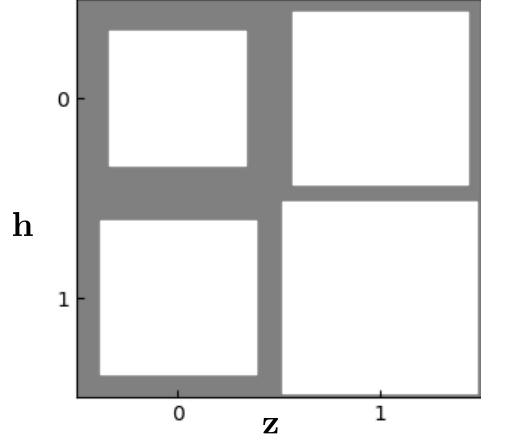} & \includegraphics[height=1.2cm]{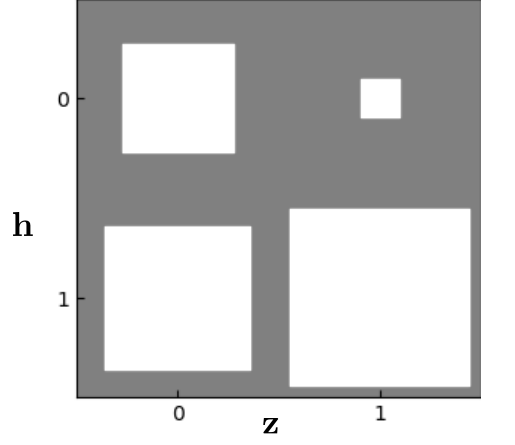} & \includegraphics[height=1.2cm]{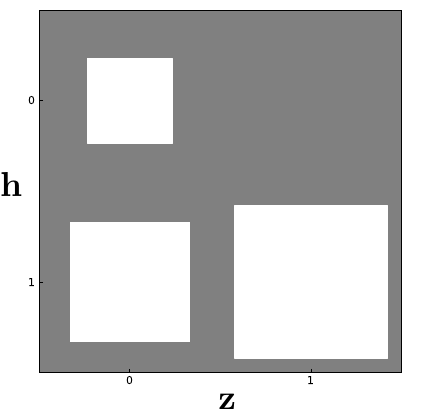} & \includegraphics[height=1.2cm]{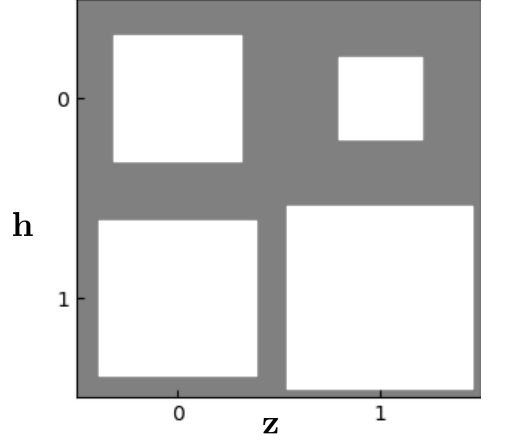} & \includegraphics[height=1.2cm]{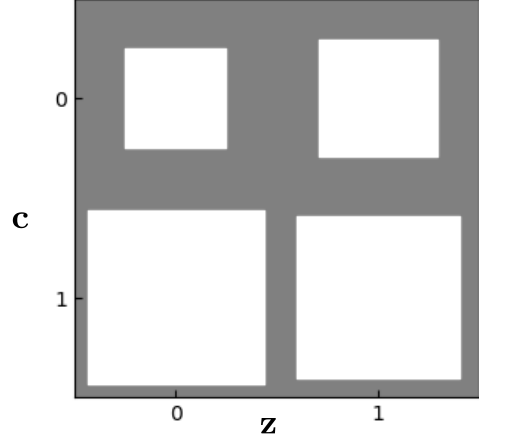} & \includegraphics[height=1.2cm]{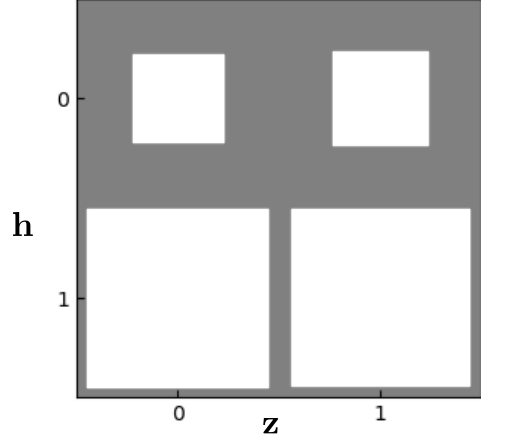} &\includegraphics[height=1.2cm]{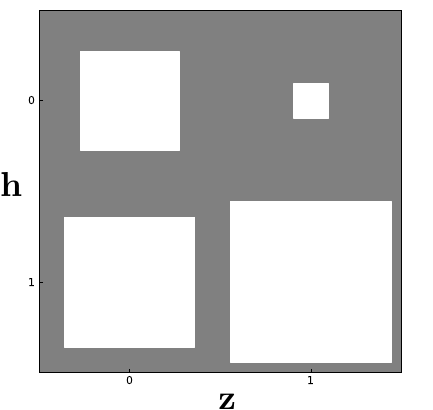}  \\
		\end{tabular}
		}
		\caption{${h}_{dim}=2$}
	\end{subfigure}
	\caption{Relative importance matrix as computed by Lasso and Random Forest regressors on DSprites dataset for ${h}_{dim}\in \left\{10,2 \right\} $ against the underlying data generating factors ${z}_{dim}= 2$ corresponding to $x,y$ positions of object.}
	\label{fig:hinton_dsp}
\end{figure}

\begin{figure}[H]
	\begin{subfigure}{1\textwidth}
		\centering
		 \resizebox{\textwidth}{!}{
		\begin{tabular}{m{1.8cm}  m{1.8cm}  m{1.8cm} m{2.2cm}| m{1.8cm}  m{1.8cm}  m{1.8cm} m{1.8cm}}
			\multicolumn{4}{c|}{\textbf{Lasso}}                                       & \multicolumn{4}{c}{\textbf{Random Forest}}                                                                                                                                                                                                                                                                                                                                                                                                \\
			\multicolumn{1}{r}{Gen-RKM}                                               & \multicolumn{1}{r}{VAE}                                                   & \multicolumn{1}{c}{$\beta \text{-VAE}$}       & \multicolumn{1}{c}{InfoGAN}                                            & \multicolumn{1}{r}{Gen-RKM}                                                               & \multicolumn{1}{r}{VAE}                                                                   & \multicolumn{1}{r}{$\beta \text{-VAE}$}       & \multicolumn{1}{c}{InfoGAN}                                                            \\
			\includegraphics[height=4cm]{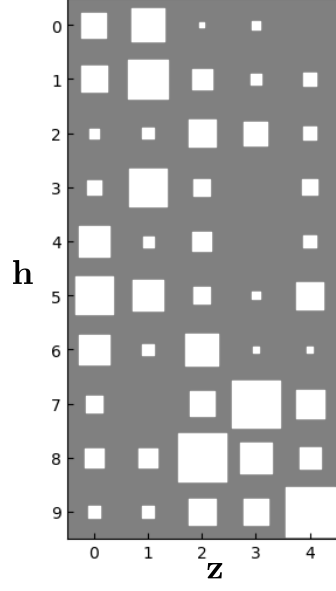} & \includegraphics[height=4cm]{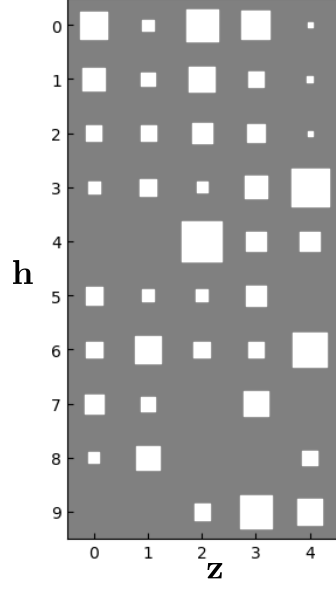} & \includegraphics[height=4cm]{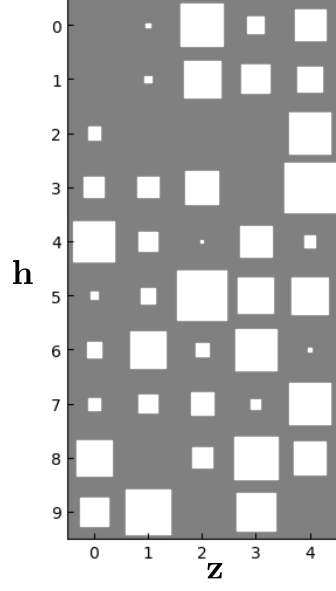} & \includegraphics[height=4cm]{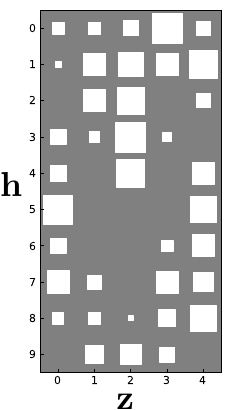} &
			\includegraphics[height=4cm]{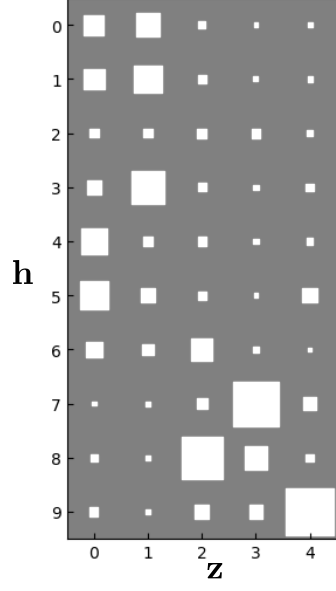} & \includegraphics[height=4cm]{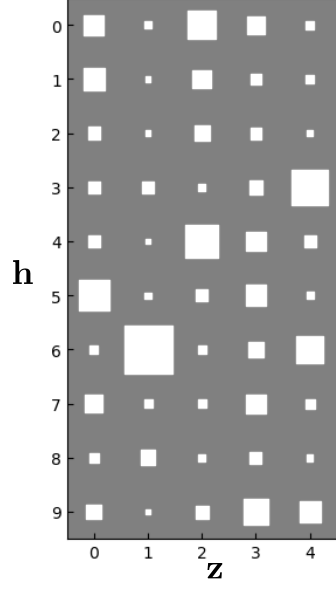} & \includegraphics[height=4cm]{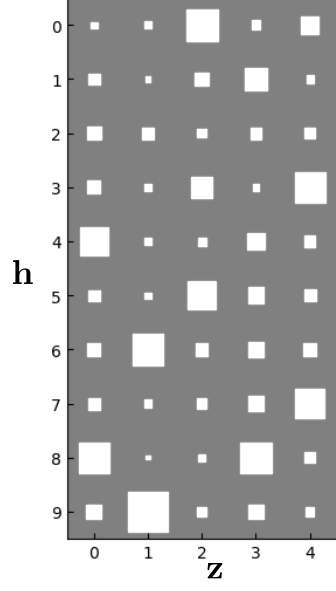}& \includegraphics[height=4cm]{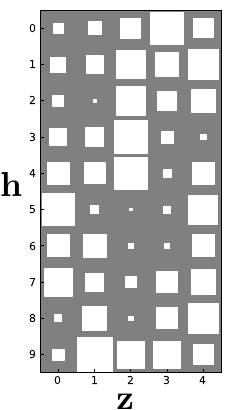}  \\
		\end{tabular}
		}
		\caption{${h}_{dim}=10$}
	\end{subfigure}
	\quad \\
	\begin{subfigure}{1\textwidth}
		\centering
		 \resizebox{\textwidth}{!}{
		\begin{tabular}{ m{1.8cm}  m{1.8cm} m{1.8cm} m{2.2cm}| m{1.8cm}  m{1.8cm} m{1.8cm} m{1.8cm}}
			\multicolumn{4}{c|}{\textbf{Lasso}}                                       & \multicolumn{4}{c}{\textbf{Random Forest}}                                                                                                                                                                                                                                                                                                                                                                                                \\
			\multicolumn{1}{r}{Gen-RKM}                                               & \multicolumn{1}{c}{VAE}                                                   & \multicolumn{1}{c}{$\beta \text{-VAE}$}     &   \multicolumn{1}{c}{InfoGAN}                                      & \multicolumn{1}{r}{Gen-RKM}                                                               & \multicolumn{1}{c}{VAE}                                                                   & \multicolumn{1}{r}{$\beta \text{-VAE}$}    &         \multicolumn{1}{c}{InfoGAN}                                                 \\
			\includegraphics[height=2cm]{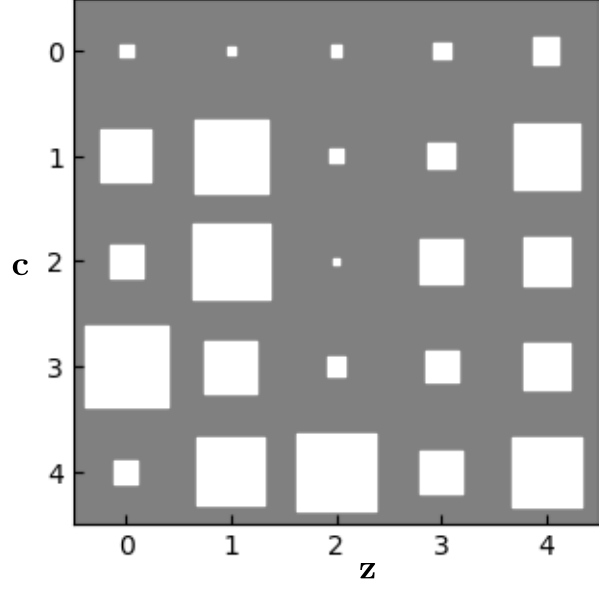} & \includegraphics[height=2cm]{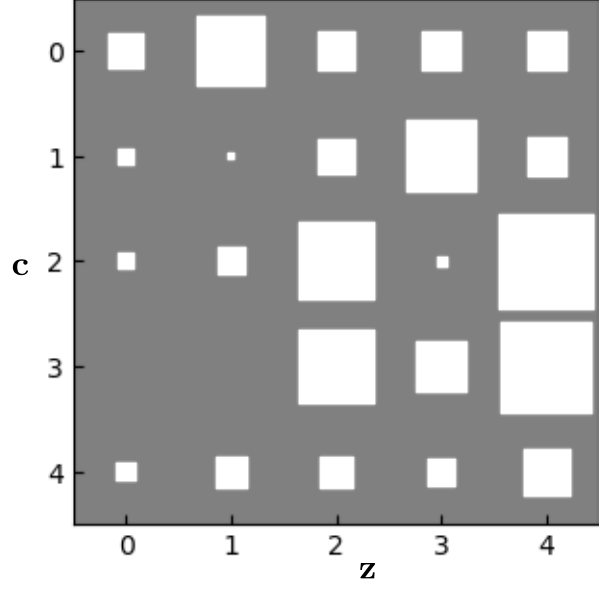} & \includegraphics[height=2cm]{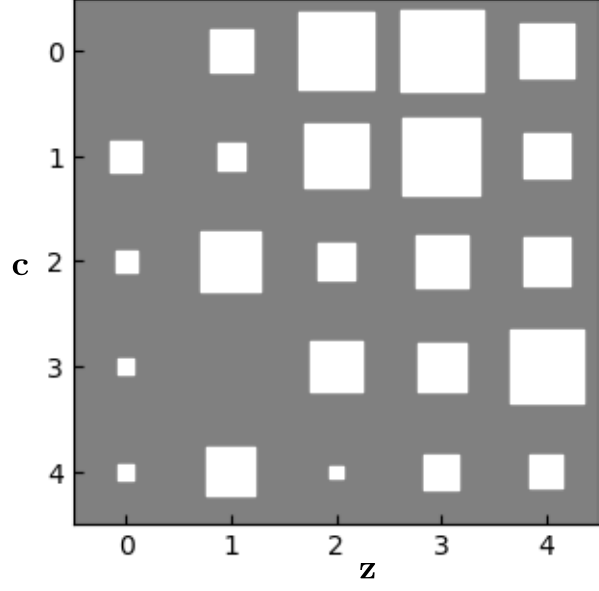} & \includegraphics[height=2cm]{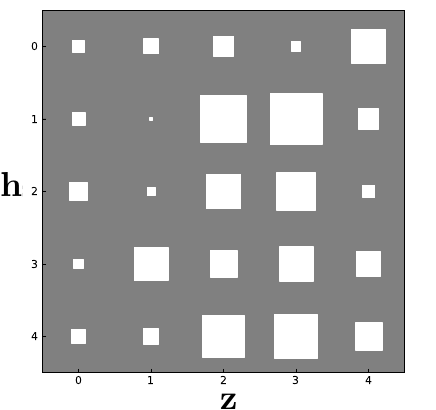} &  \includegraphics[height=2cm]{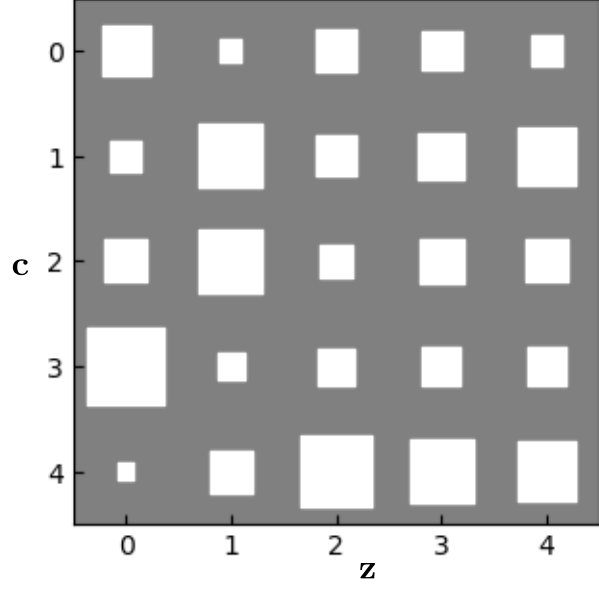} & \includegraphics[height=2cm]{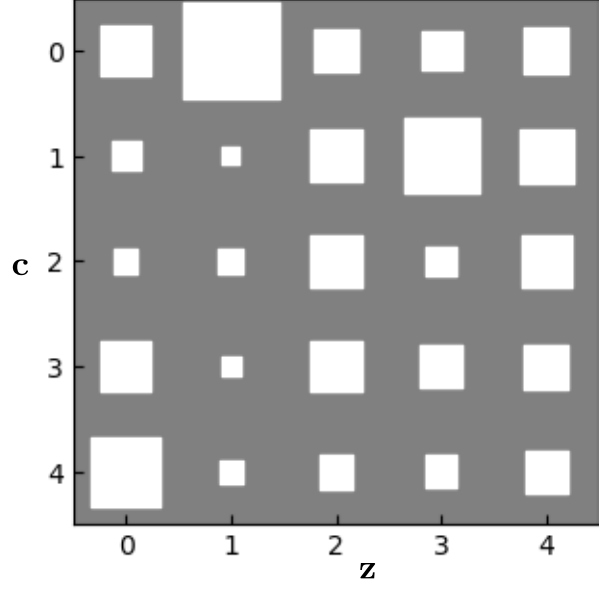} & \includegraphics[height=2cm]{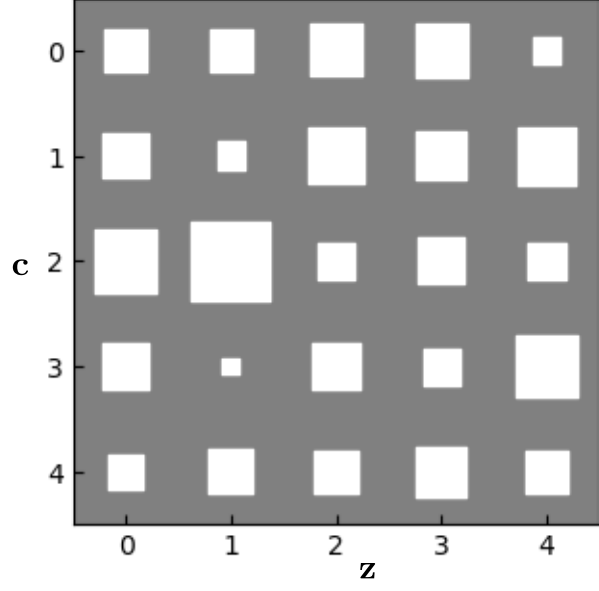} & \includegraphics[height=2cm]{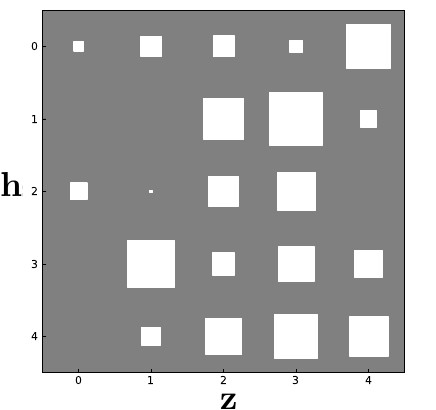}\\
		\end{tabular}
		}
		\caption{${h}_{dim}=5$}
	\end{subfigure}
	\caption{Relative importance matrix as computed by Lasso and Random Forest regaressors on Teapot dataset for ${h}_{dim}\in \left\{10,5 \right\} $ against the underlying data generating factors ${z}_{dim}= 5$ corresponding to azimuth, elevation and colors red, green and blue of the teapot object.}
	\label{fig:hinton_tea}
\end{figure}

\section{Further empirical results}\label{appen:further_results}

\begin{table}[H]
\caption{Training time per epoch comparisons (in seconds with standard deviation over 10 epochs) on MNIST and CelebA datasets. The architecture is the same as shown in Table \ref{tab:arch} below with mini-batch size 100 and batch-size 2000. In both the cases Info-GAN is the most computationally expensive due to the additional auxiliary network and two backward passes per iteration. $\beta$-TCVAE has the second worst computation times due to relatively more complicated ELBO objective. VAE is marginally better incase of MNIST whereas the Gen-RKM outperforms incase of CelebA. This could be due to significantly large number of parameters for CelebA architecture which increases the computational burden of VAE. However, due to the fixed computational cost of eigendecomposition (for fixed mini-batch size), the latent variables in Gen-RKM are computed with this fixed cost.}
\label{tab:train_times}
\centering
    \begin{tabular}{cccccc}
    \toprule
Dataset & Gen-RKM  & VAE & Info-GAN & $\beta$-TCVAE \\ \midrule
 MNIST   &    0.275 ($\pm$0.042)  & \textbf{0.223} ($\pm$0.013) & 0.372 ($\pm$0.044) & 0.318 ($\pm$0.031) \\
 CelebA  &  \textbf{4.274} ($\pm$0.147)   & 4.308 ($\pm$0.112) & 5.815 ($\pm$0.131)  & 5.201 ($\pm$0.152)  \\ \bottomrule
    \end{tabular}
\end{table}

\subsection{Illustration on toy example using a Gaussian kernel}

Here we demonstrate the application of Gen-RKM / Kernel PCA using a Gaussian kernel with $\sigma = 0.5$ on a 3 mode Gaussian dataset. The dataset is shown in Figure~\ref{fig:toy1} together with the first 5 Principal Components (PCs) of the latent space. The method looks for PCs that explain the most variance. One can see that moving along first component in latent space correspond to changing classes 3 $\rightarrow$ 2 $\rightarrow$ 1, whereas moving along the second component corresponds to changing classes 2 $\rightarrow$ 3 $\rightarrow$ 1. For PCs 3 to 5, the model shows disentanglement of the 3 classes, i.e. each Gaussian cluster is mapped to a specific component.
Moving along one of these components only changes the \emph{within class} variation. This behaviour is further confirmed by the experiment on Figure~\ref{fig:toy2}.
Here we visualize again the dataset where now the color corresponds to the value of the datapoint in latent space.

\begin{figure}[ht]
	\centering
	\begin{subfigure}[b]{0.32\textwidth}
		\centering
\includegraphics[height=3cm, width=1\textwidth]{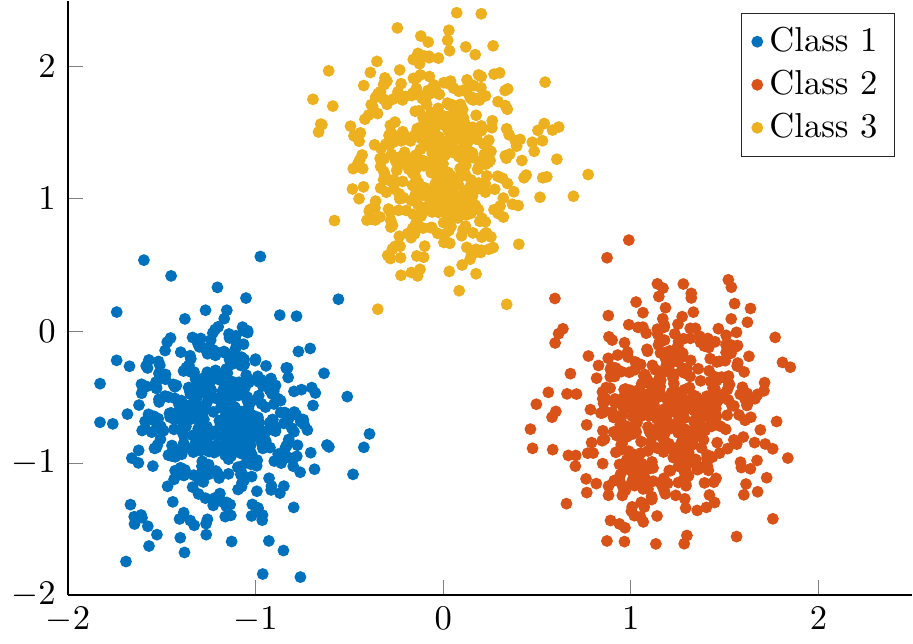}
\caption{Input data}
	\end{subfigure}
	\begin{subfigure}[b]{0.32\textwidth}
		\centering
\includegraphics[height=3cm, width=1\textwidth]{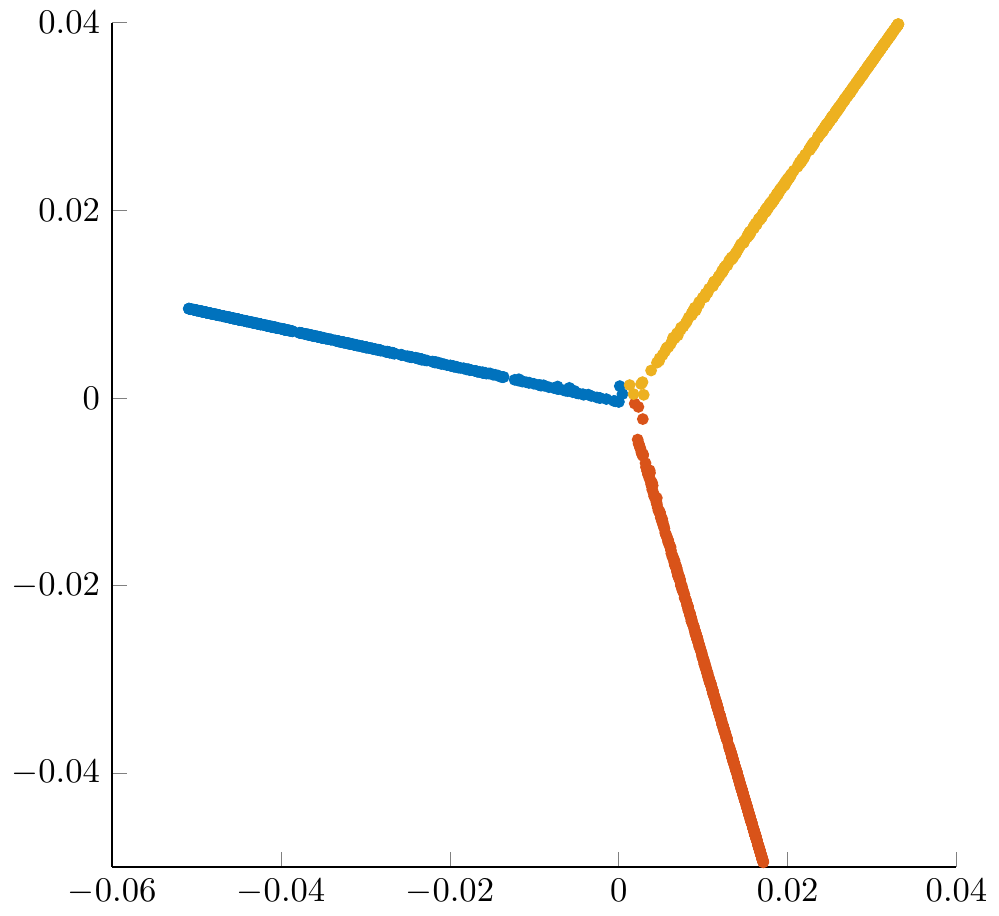}
\caption{PCs 1 \& 2}
	\end{subfigure}
	\begin{subfigure}[b]{0.32\textwidth}
		\centering
\includegraphics[height=3cm, width=1\textwidth]{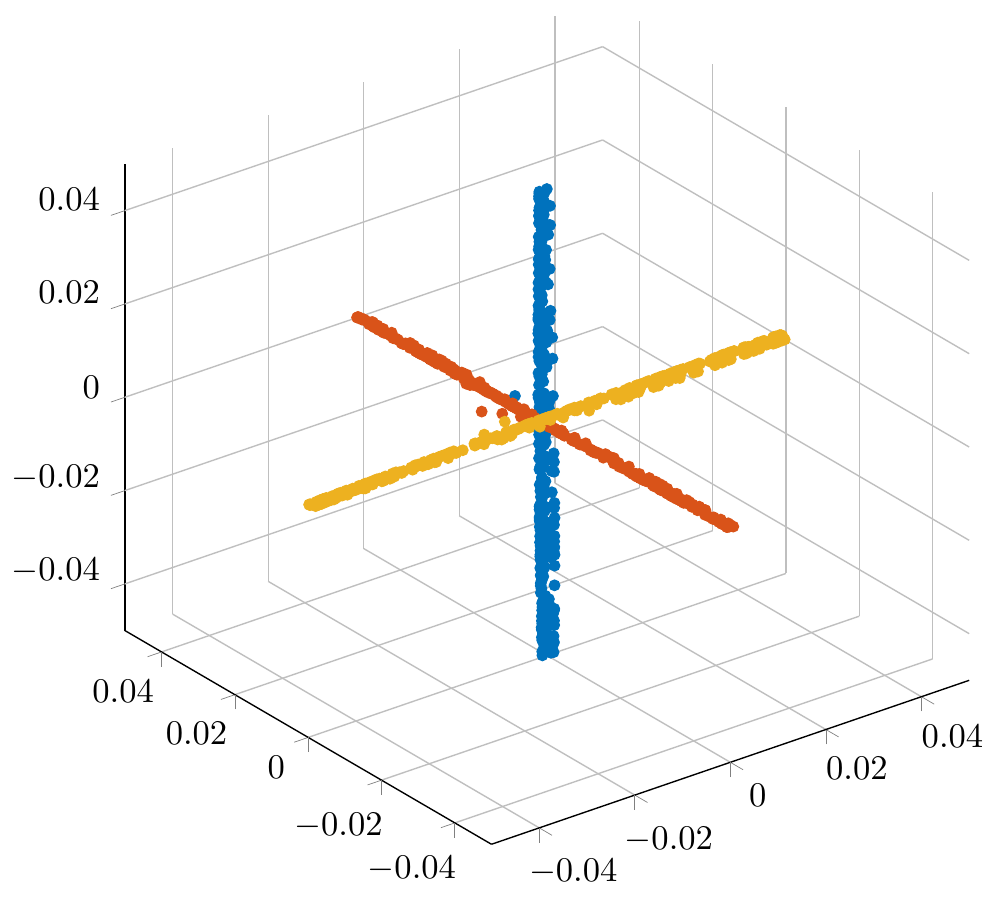}
\caption{PCs 3,4 \& 5}
	\end{subfigure}
		\caption{Visualization of the toy dataset together with the first 5 Principal Components (PCs) of the latent space of the Gen-RKM model. }\label{fig:toy1}
\end{figure}

\begin{figure}[ht]
	\centering
	\begin{subfigure}[b]{0.32\textwidth}
		\centering
		\includegraphics[width=1\textwidth]{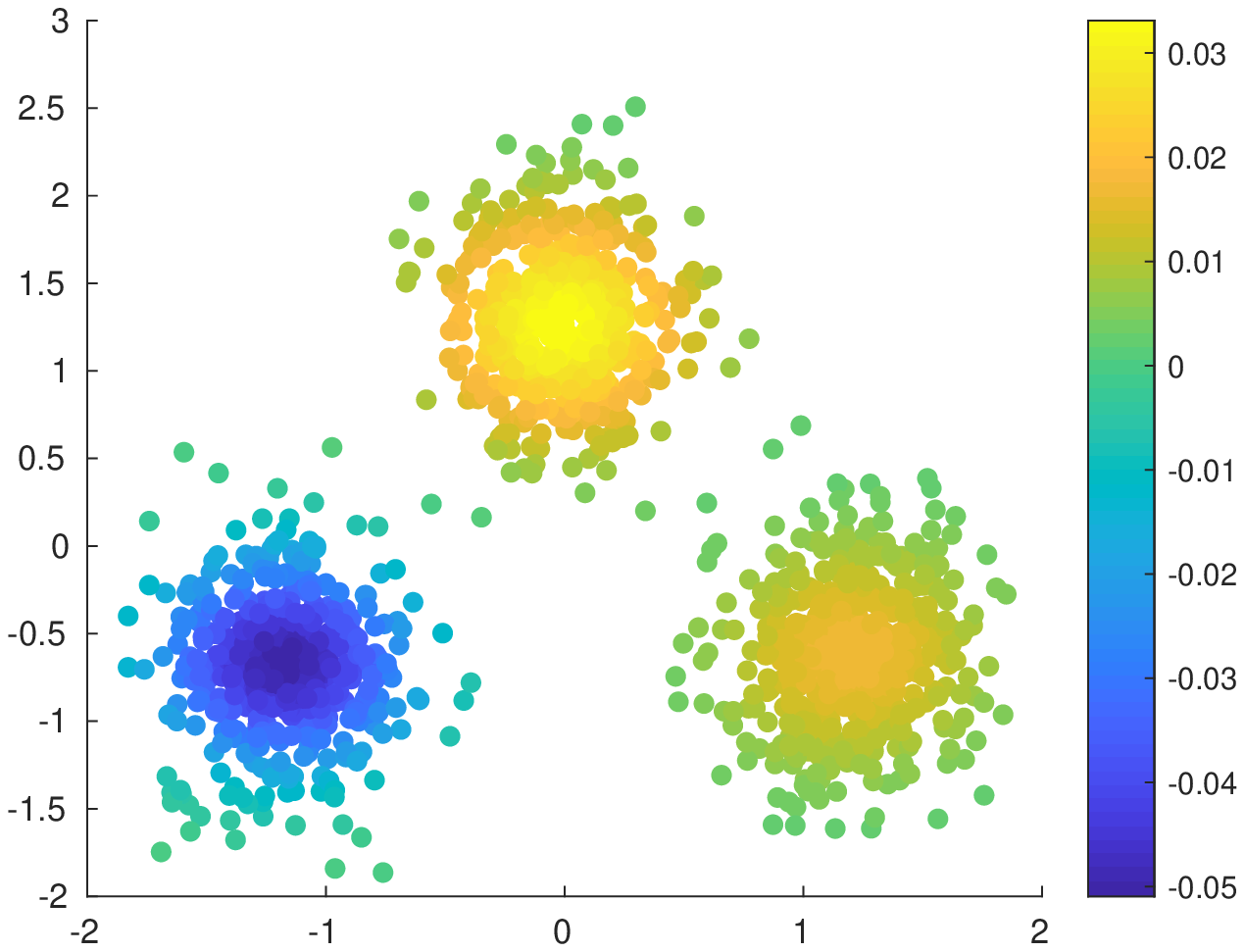}
		\caption{Latent dim. 1}
	\end{subfigure}
	\begin{subfigure}[b]{0.32\textwidth}
		\centering
		\includegraphics[width=1\textwidth]{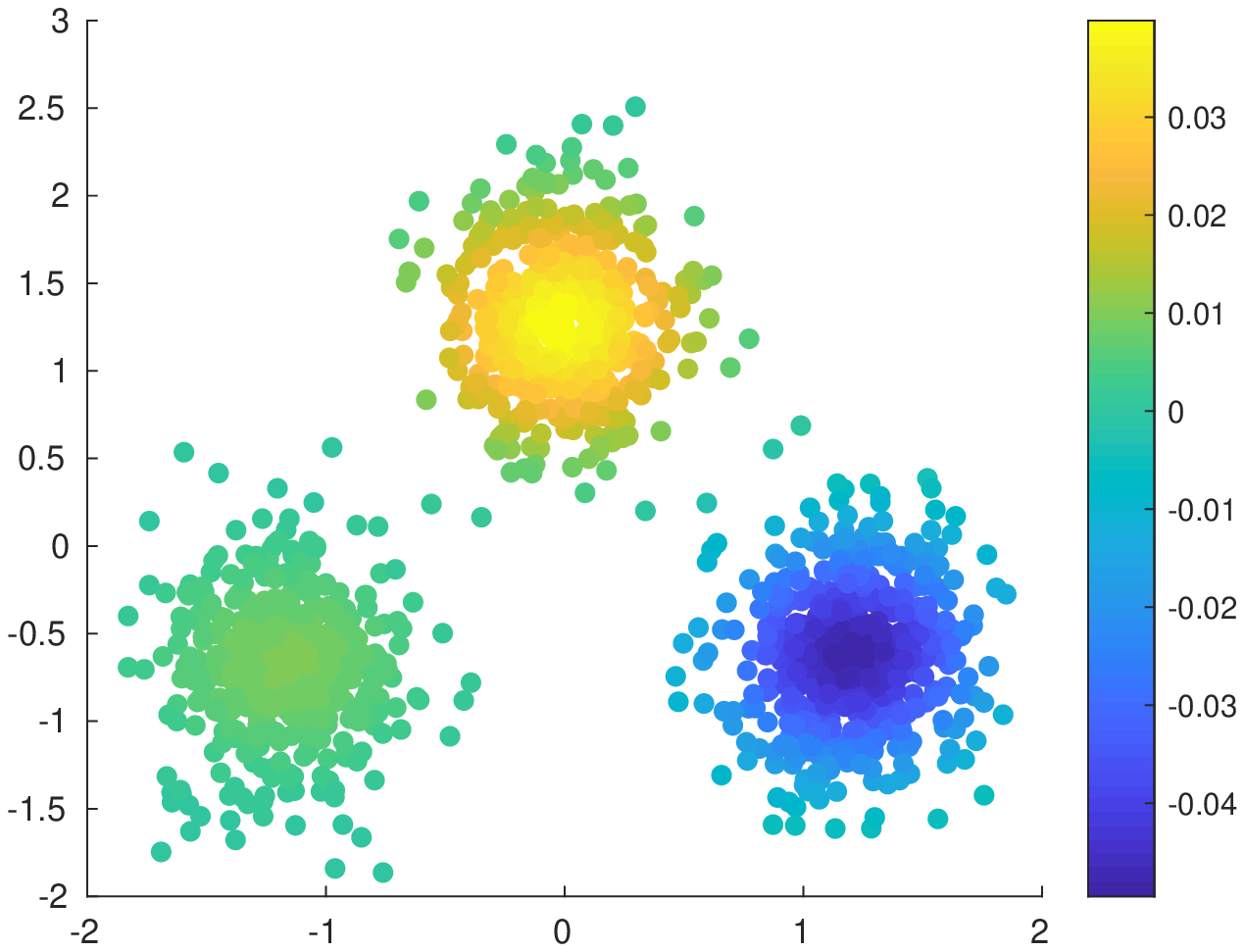}
		\caption{Latent dim. 2}
	\end{subfigure}
	\begin{subfigure}[b]{0.32\textwidth}
		\centering
		\includegraphics[width=1\textwidth]{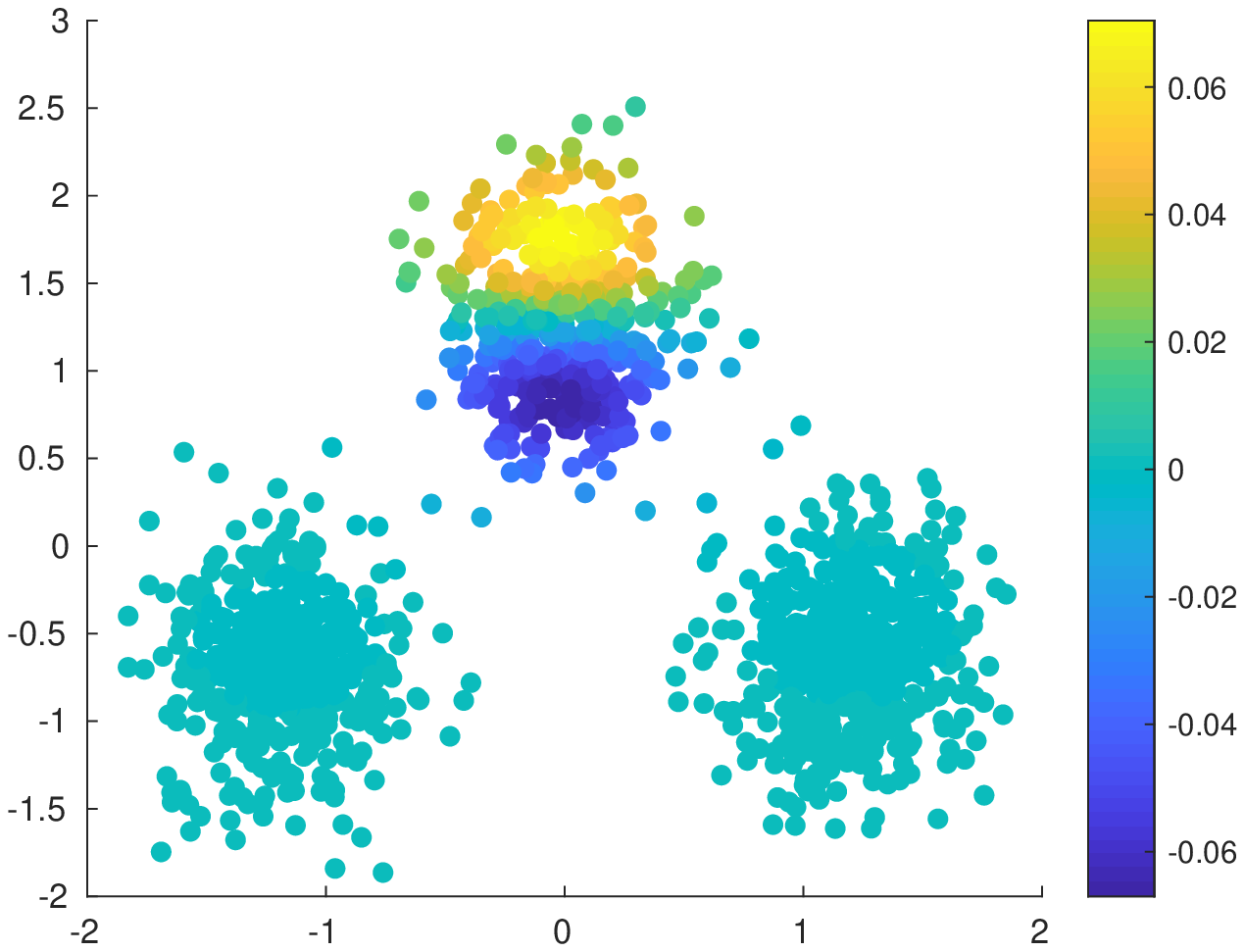}
		\caption{Latent dim. 3}
	\end{subfigure}
	\begin{subfigure}[b]{0.32\textwidth}
		\centering
		\includegraphics[width=1\textwidth]{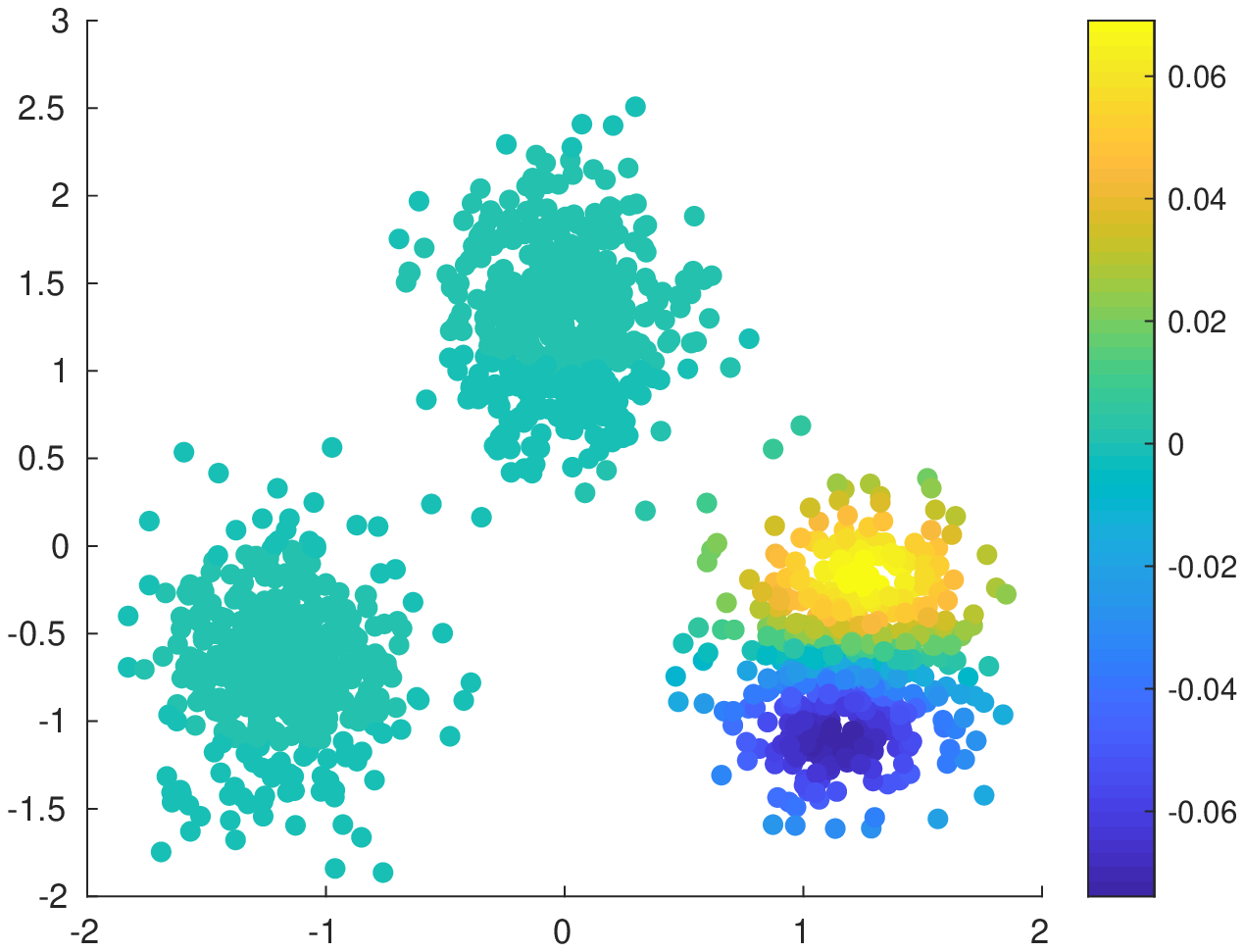}
		\caption{Latent dim. 4}
	\end{subfigure}
	\begin{subfigure}[b]{0.32\textwidth}
		\centering
		\includegraphics[width=1\textwidth]{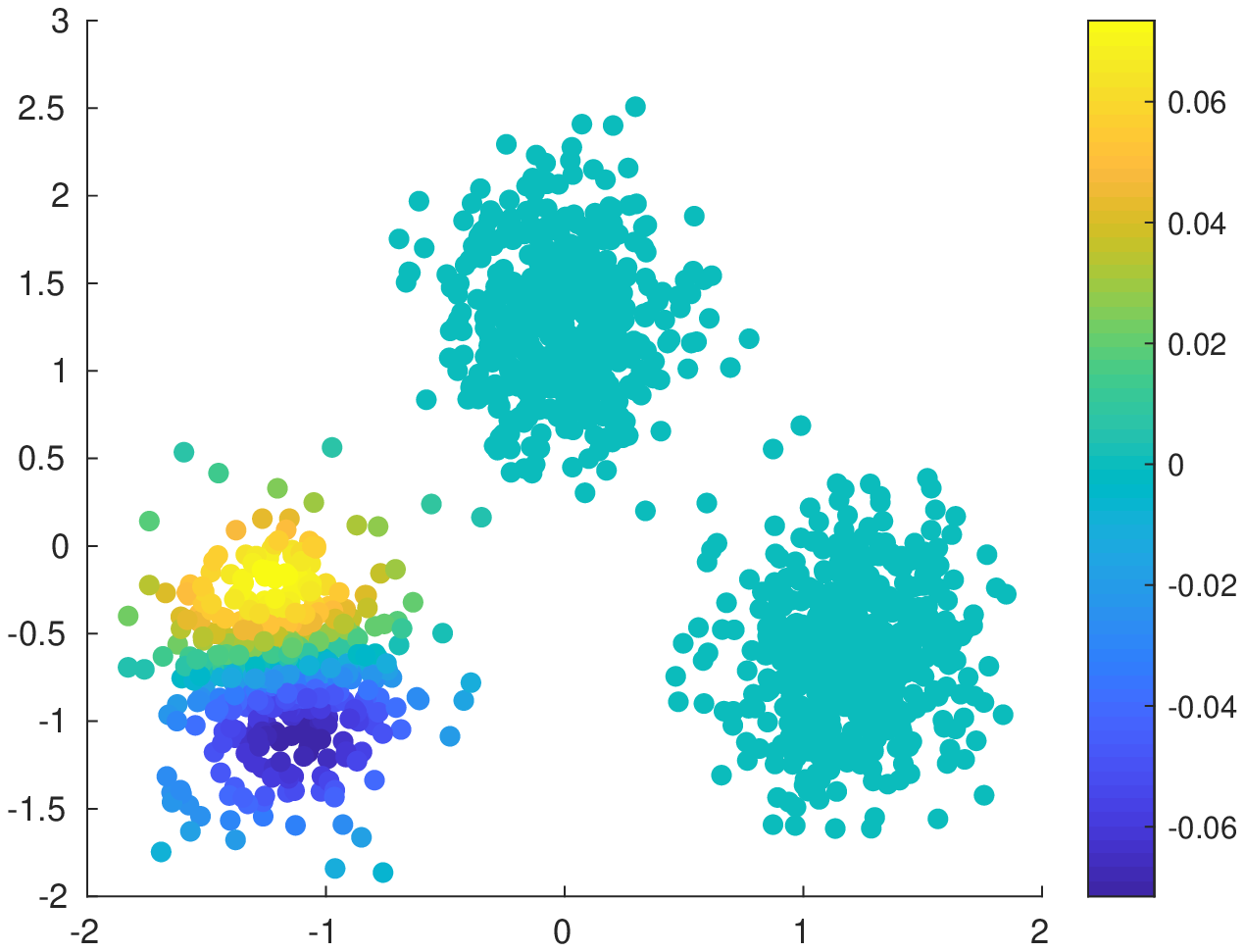}
		\caption{Latent dim. 5}
	\end{subfigure}
	\caption{Visualization of the traversals along the Principal components. Here the color corresponds to the value of the datapoint in latent space.}\label{fig:toy2}
\end{figure}

\begin{figure}[ht]
	\centering
	\begin{tabular}{c c c}
		& Gen-RKM                                             & VAE                                                 \\
		\rotatebox{90}{\qquad Reconstructions}    & \includegraphics[height=4.5cm]{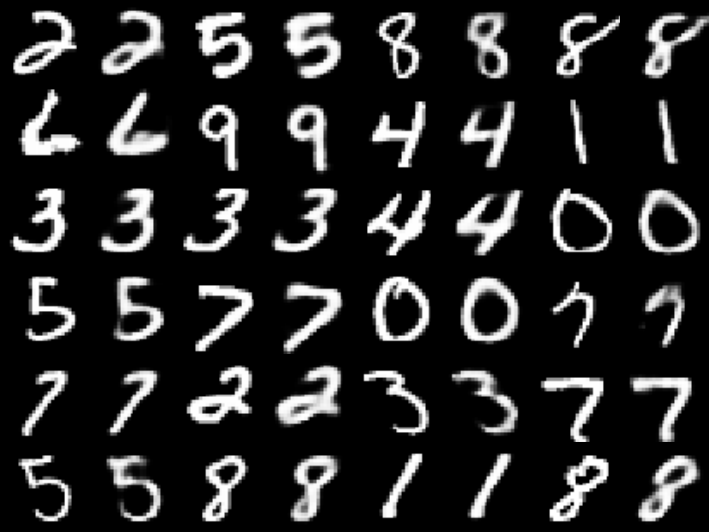} & \includegraphics[height=4.5cm]{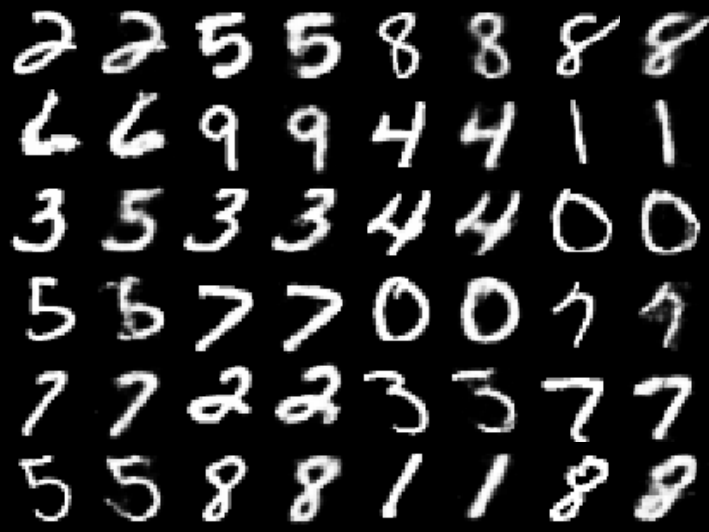} \\
		\rotatebox{90}{ Random Generation}  & \includegraphics[height=4.5cm]{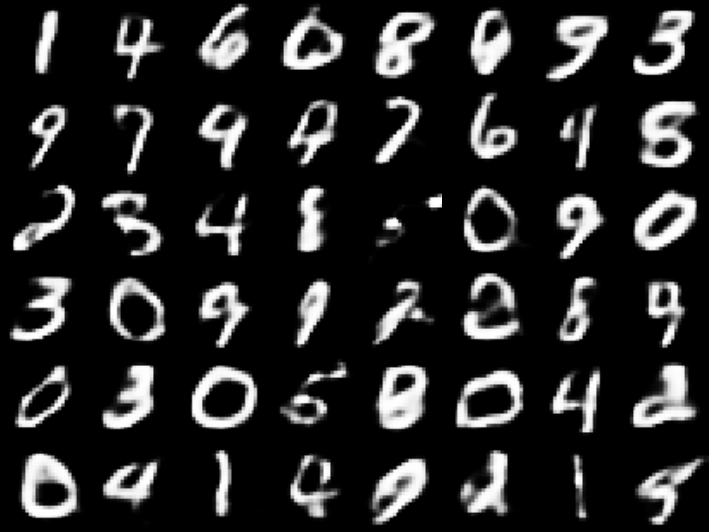}  & \includegraphics[height=4.5cm]{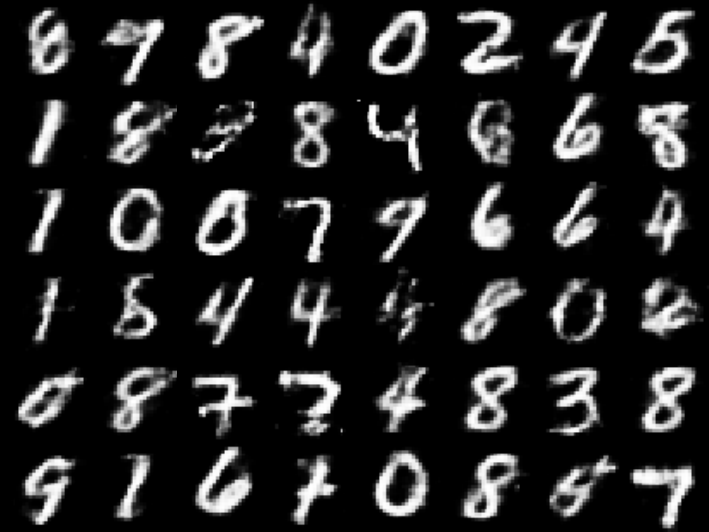}  \\
		\rotatebox{90}{\qquad Reconstructions}   & \includegraphics[height=4.5cm]{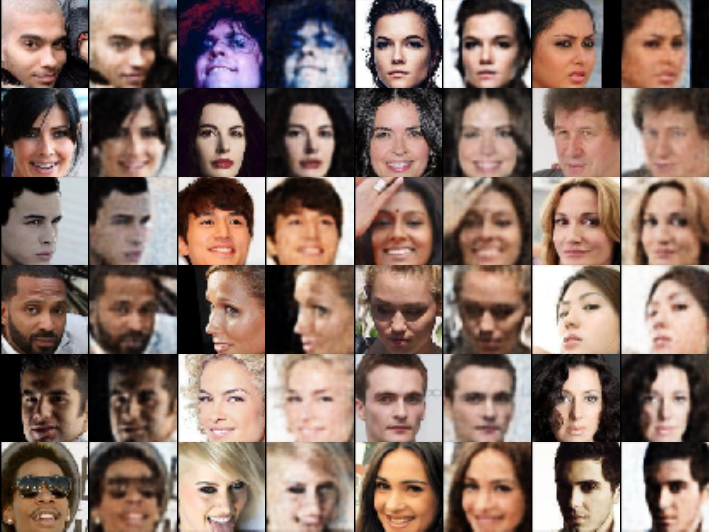} & \includegraphics[height=4.5cm]{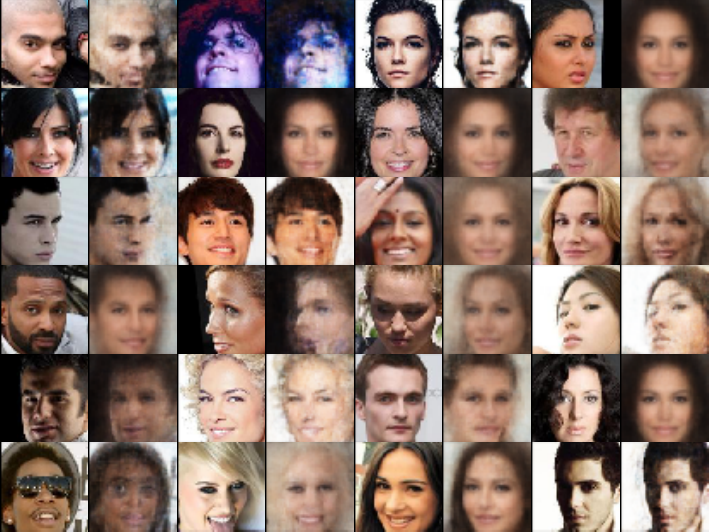} \\
		\rotatebox{90}{ Random Generation} & \includegraphics[height=4.5cm]{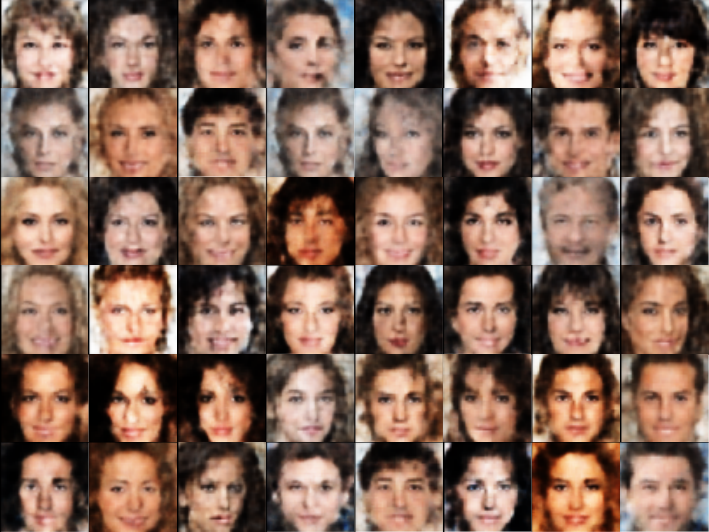} & \includegraphics[height=4.5cm]{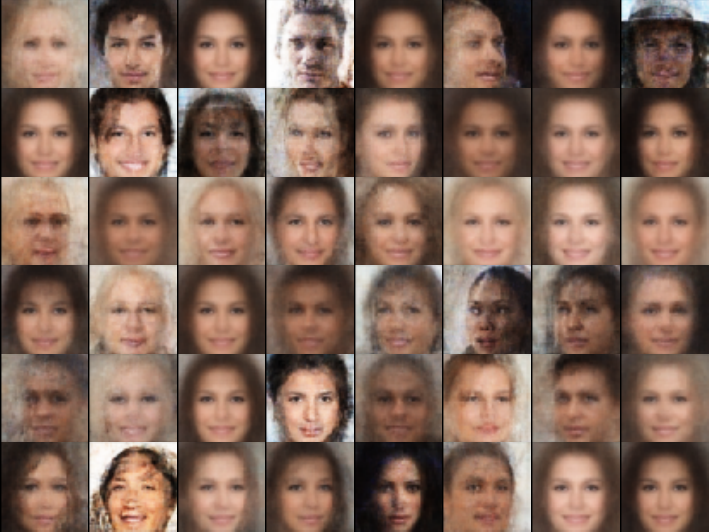}
	\end{tabular}
	\caption{Comparing Gen-RKM and standard VAE for reconstruction and generation quality. In reconstruction MNIST and reconstruction CelebA, uneven columns correspond to the original image, even columns to the reconstructed image.}
	\label{tab:ReconstructionComparison}
\end{figure}
\begin{figure}[ht]
\begin{subfigure}{1\textwidth}
\centering
\includegraphics[trim={0.5cm 0.5cm 0.5cm 0.5cm},clip, width=0.4\textwidth]{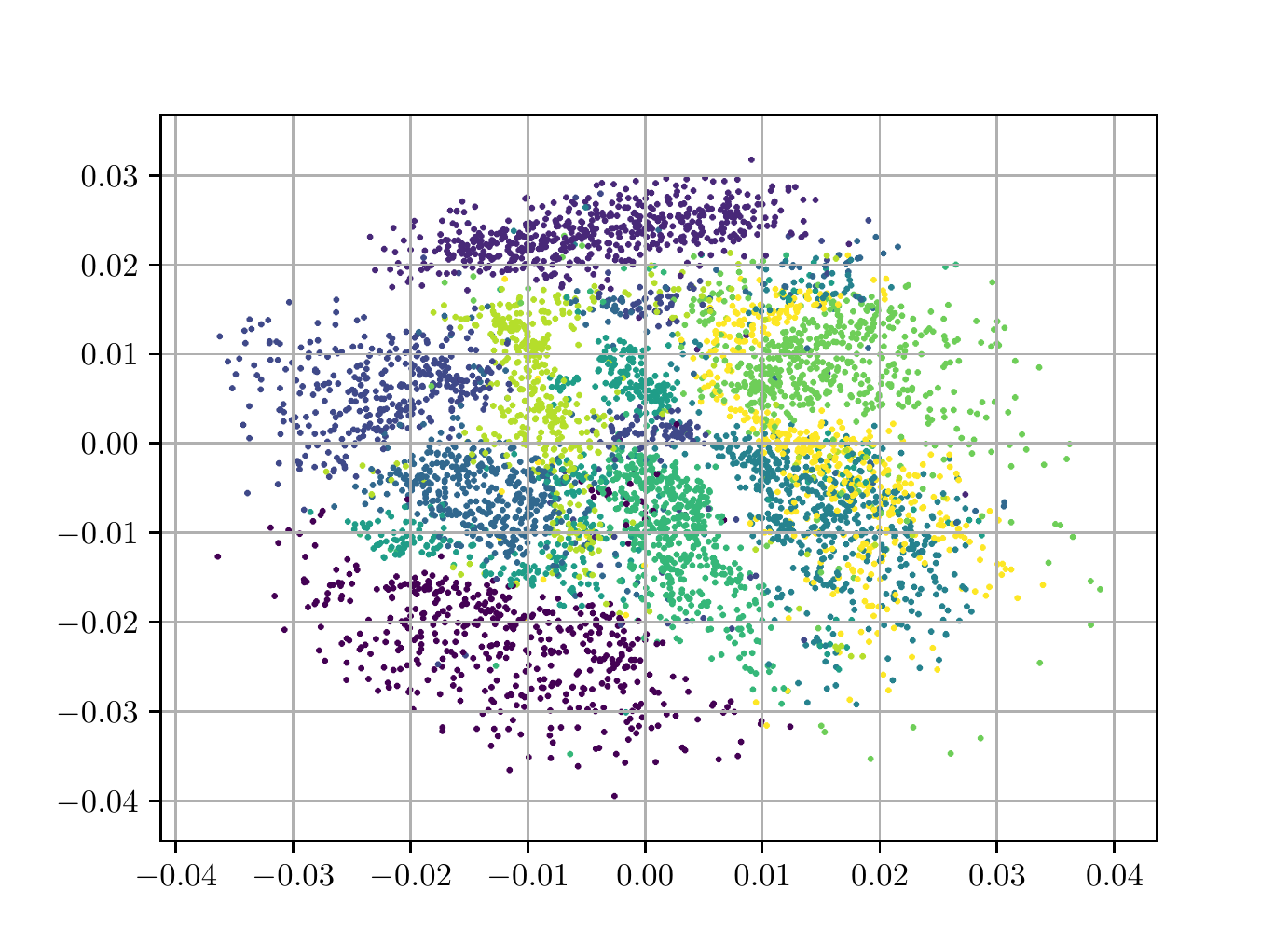}
\caption{Uniform grid $[-0.03,0.03]^{2}$ over latent space.}
\end{subfigure}

\begin{subfigure}{1\textwidth}
\centering
\includegraphics[trim={2cm 2cm 2cm 2cm},clip, width=0.9\textwidth]{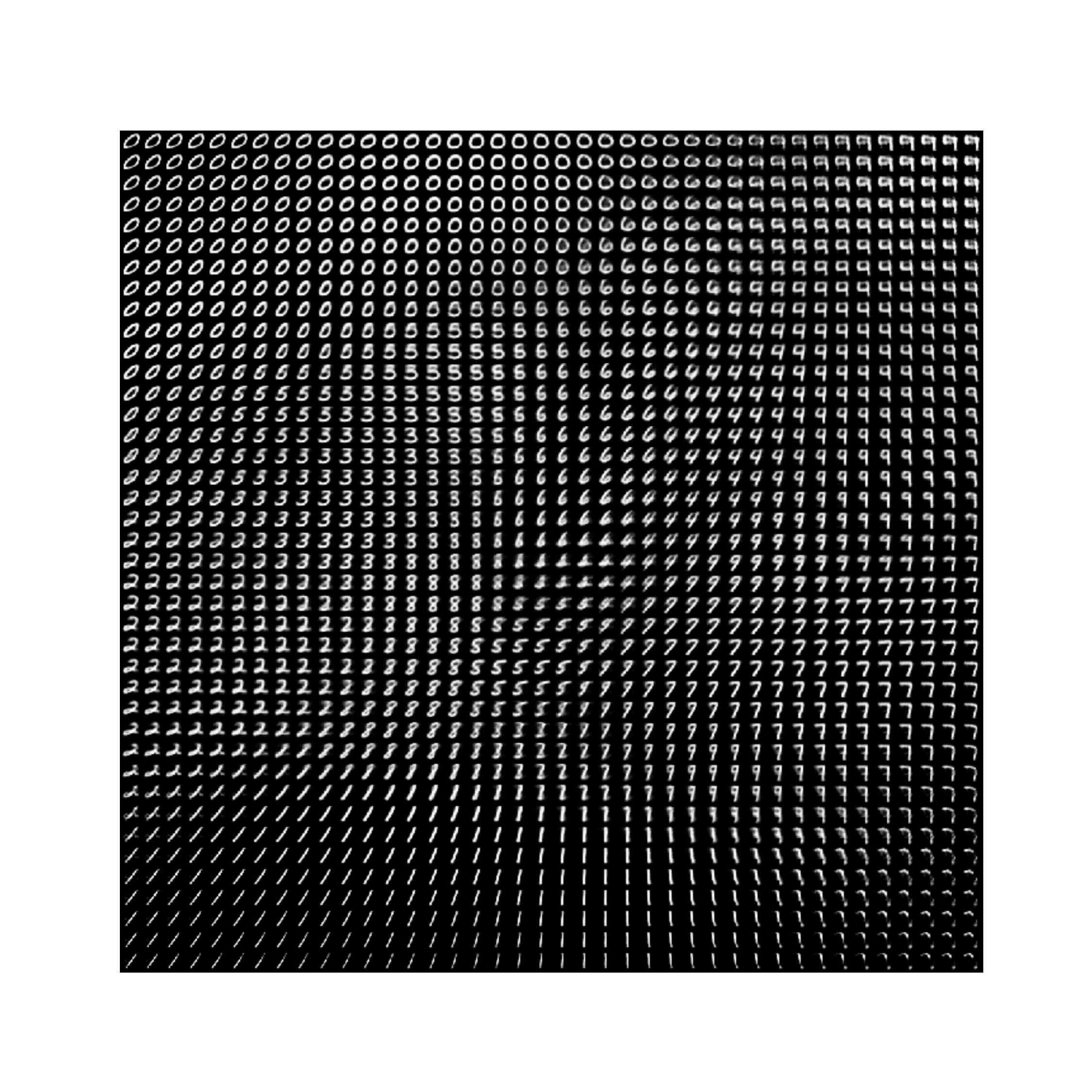}
\caption{Generated images from latent vectors obtained from the uniform grid in the latent space depicting smoothness of latent space.}
\end{subfigure}
\caption{MNIST: Latent space exploration.}
\label{fig:mnist_unif}
\end{figure}{}

\end{document}